%% file: main.tex
\PassOptionsToPackage{table}{xcolor}
\documentclass{article}

\usepackage[final]{neurips_2024}

\usepackage[utf8]{inputenc} 
\usepackage[T1]{fontenc}    
\usepackage{url}            
\usepackage{booktabs}       
\usepackage{amsfonts}       
\usepackage{nicefrac}       
\usepackage{microtype}      
\usepackage{xcolor}         

\usepackage{mwe}

\usepackage{multirow}

\usepackage{booktabs}

\usepackage{microtype}
\usepackage{graphicx}
\usepackage{subfigure}
\usepackage{xfrac}
\usepackage[breakable, most]{tcolorbox}
\usepackage{tikz-cd} 
\usepackage{sidecap}
\sidecaptionvpos{figure}{t}
\usepackage{wrapfig}

\usepackage[hidelinks]{hyperref}

\usepackage{amsmath}
\usepackage{amssymb}
\usepackage{mathtools}
\usepackage{amsthm}
\usepackage{algorithm}
\usepackage{algorithmic}
\usepackage[utf8]{inputenc} 
\usepackage[T1]{fontenc}    
\usepackage{amssymb, amsmath}
\usepackage{amsfonts}       
\usepackage{nicefrac}       
\usepackage{microtype}      
\usepackage{subfigure}
\usepackage{caption}
\usepackage{multirow}
\usepackage{float}
\usepackage{bm}
\usepackage{pifont}
\usepackage{framed}
\usepackage{booktabs}
\usepackage{enumitem}
\usepackage{scalefnt}
\newcommand\smaller[2][0.85]{{\scalefont{#1}#2}}

\usepackage[capitalize]{cleveref}
\creflabelformat{equation}{#2\textup{#1}#3}

\theoremstyle{plain}
\newtheorem{theorem}{Theorem}[section]
\newtheorem{proposition}[theorem]{Proposition}
\newtheorem{lemma}[theorem]{Lemma}

\theoremstyle{definition}
\newtheorem{definition}[theorem]{Definition}

\theoremstyle{remark}

\newcommand{\R}[0]{\mathbb{R}}
\newcommand{\N}[0]{\mathcal{N}}
\newcommand{\relu}[1]{\mathrm{ReLU}\left(#1\right)}
\renewcommand{\vec}[1]{\mathbf{#1}}
\newcommand{\mat}[1]{\mathbf{#1}}
\newcommand{\w}[0]{\vec{w}}
\newcommand{\x}[0]{\vec{x}}
\newcommand{\y}[0]{\vec{y}}
\newcommand{\J}[0]{\mat{J}}
\newcommand{\Hess}[0]{\mat{H}}
\newcommand{\dif}[0]{\mathrm{d}}
\newcommand{\T}[0]{^{\top}}

\newcommand{\M}[0]{\mathrm{M}}
\newcommand{\m}[0]{\mathfrak{m}}
\newcommand{\identity}{\mat{I}}
\newcommand{\ggn}{\textsc{ggn}}
\newcommand{\ntk}{\textsc{ntk}}
\newcommand{\Ker}[1]{\mathrm{ker}(#1)}
\newcommand{\im}[1]{\mathrm{im}(#1)}
\newcommand{\wh}{\hat{\w}}

\definecolor{boxcolor}{HTML}{27778d}

\newtcolorbox{mybox}[2][]{
    fonttitle=\itshape,
    title={#2},
    enhanced,
    colframe=boxcolor,
    coltitle=boxcolor,
    colback=white,
    boxed title style={opacityback=0,colframe=white,size=fbox,arc=0mm},
    attach boxed title to top left={yshift=-\tcboxedtitleheight/2,xshift=4mm}
}

\usepackage{contour}
\usepackage[normalem]{ulem}

\contourlength{0.8pt}

\renewcommand{\underline}[1]{%
  \uline{\phantom{#1}}%
  \llap{\contour{white}{#1}}%
}

\renewcommand{\paragraph}[1]{\textbf{#1}~~}

\newcommand{\first}[1]{\textbf{#1}}

\title{Reparameterization invariance in approximate Bayesian inference}

\author{%
  Hrittik Roy$^\dagger$, Marco Miani$^\dagger$ \\
  Technical University of Denmark \\
  \texttt{\{hroy, mmia\}@dtu.dk} \\
  \And
  Carl Henrik Ek \\
  University of Cambridge,\\
  Karolinska Institutet \\
  \texttt{che29@cam.ac.uk}
  \And
  Philipp Hennig, Marvin Pförtner, Lukas Tatzel \\
  University of Tübingen,
  Tübingen AI Center \\
  \texttt{\{philipp.hennig, lukas.tatzel,} \\\texttt{marvin.pfoertner\}@uni-tuebingen.de}
  \And
  Søren Hauberg \\
  Technical University of Denmark \\
  \texttt{sohau@dtu.dk}
}

\begin{document}
\renewcommand{\thefootnote}{$\dagger$} 
\footnotetext{Equal contribution authors listed in random order.}
\vspace{-3mm}
\maketitle

\vspace{-5mm}
\begin{abstract}
    Current approximate posteriors in Bayesian neural networks (\textsc{bnn}s) exhibit a crucial limitation: they fail to maintain invariance under reparameterization, i.e.\@ \textsc{bnn}s assign different posterior densities to different parametrizations of identical functions. This creates a fundamental flaw in the application of Bayesian principles as it breaks the correspondence between uncertainty over the parameters with uncertainty over the parametrized function.
    In this paper, we investigate this issue in the context of the increasingly popular linearized Laplace approximation. Specifically, it has been observed that linearized predictives alleviate the common underfitting problems of the Laplace approximation. We develop a new geometric view of reparametrizations from which we explain the success of linearization. Moreover, we demonstrate that these reparameterization invariance properties can be extended to the original neural network predictive using a Riemannian diffusion process giving a straightforward algorithm for approximate posterior sampling, which empirically improves posterior fit. 
\end{abstract}

\section{Introduction} 
Bayesian deep learning has not seen the same degree of success as deep learning in general. Theoretically, Bayesian posteriors \emph{should} be superior to point estimates \citep{devroye1996probabilistic}, but the practical benefits of having the posterior are all too often not significant enough to justify their additional computational burden. This has raised the question if we even \emph{should} attempt to estimate full posteriors of all network parameters \citep{sharma2023bayesian}.

As an example, consider the Laplace approximation \citep{mackay1992laplace}, which places a Gaussian in weight space through a second-order Taylor expansion of the log-posterior. When applied to neural networks, this is known to significantly \emph{underfit} and assign significant probability mass to functions that fail to fit the training data \citep{lawrence2001variational, immer2021scalable}. Fig.~\ref{fig:laplace_fails} (top-left) exemplifies this failure mode for a small regression problem. Interestingly, this behavior is rarely observed outside neural network models, and the failure appears linked to Bayesian deep learning. 

Recently, the \emph{linearized Laplace approximation (\textsc{lla})} has been shown to significantly improve on Laplace's approximation through an additional linearization of the neural network \citep{immer2021improving, khan2019approximate}. We are unaware of any theoretical justification for this rather counterintuitive result: \emph{why would an additional degree of approximation \underline{improve} the posterior fit?}

We will show that the failures of Bayesian deep learning can partly be explained by insufficient handling of \emph{reparameterizations} of network weights, while the \textsc{lla} achieves infinitesimal invariance to reparameterizations. To motivate, consider the simple network (Fig.~\ref{fig:relu})
\begin{align}
  f(x) &= w_1 \relu{w_2\, x}; \qquad f: \R \rightarrow \R.
  \label{eq:relu}
\end{align}
This can be \emph{reparametrized} to form the same function realization from different weights as $f(x) = \sfrac{w_1}{\alpha}\, \relu{\alpha\, w_2\, x}$ for any $\alpha > 0$. That is, the weight-pairs $(w_1, w_2)$ and $(\sfrac{w_1}{\alpha}, \alpha w_2)$ correspond to the same function even if the weights are different (Fig.~\ref{fig:relu}, center).

\begin{SCfigure}[0.7][t]
    \centering
    \includegraphics[width=0.6\linewidth]{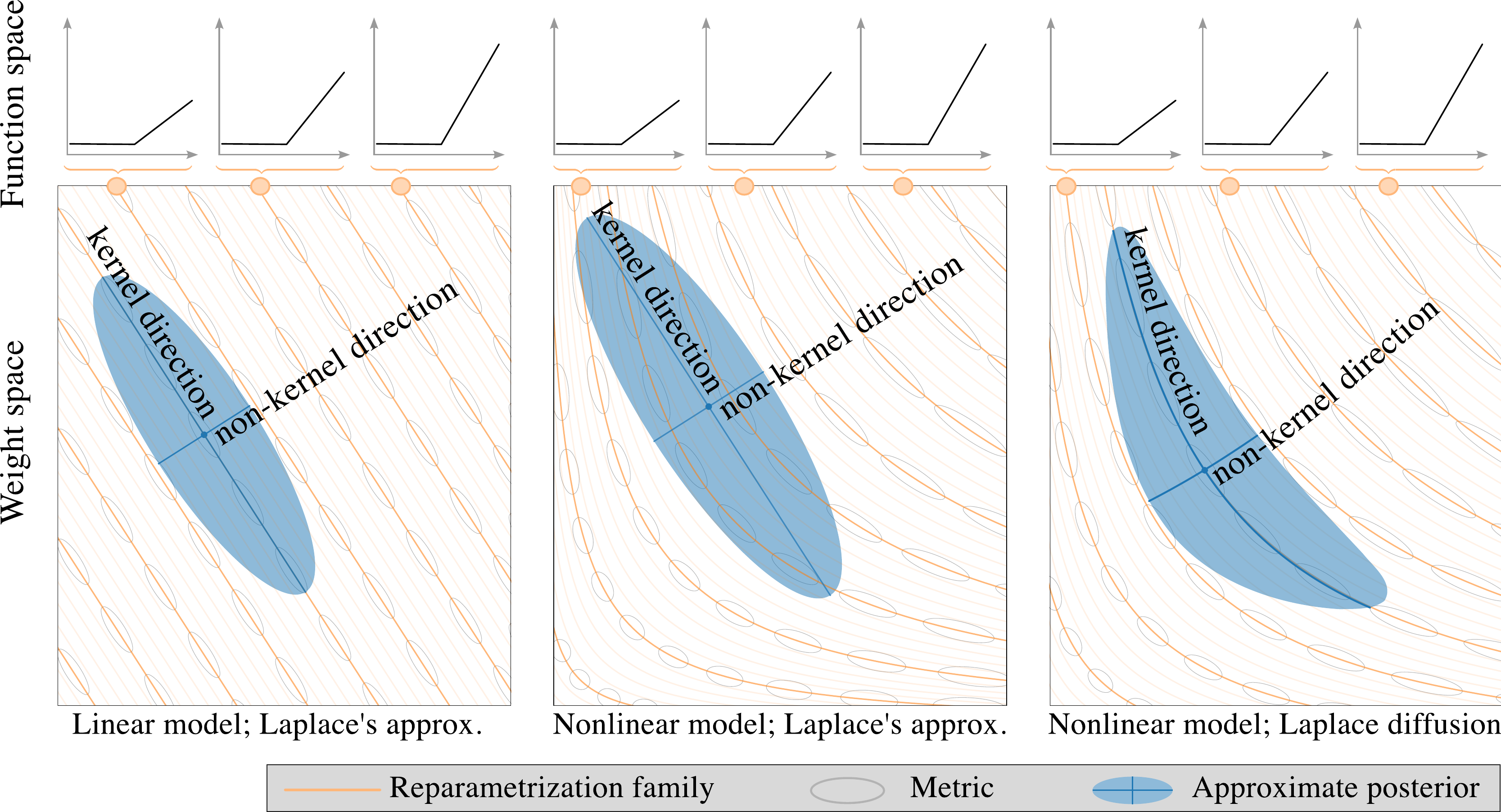}
    \caption{The \emph{weight space} of a neural network (Eq.~\ref{eq:relu}) overparametrizes the associated \emph{function space}. This induces families (orange) of weights corresponding to the same functions. Model linearization (left) linearizes these families. In nonlinear models, Gaussian weight distributions (center) do not adapt to the families, while our geometric diffusion (right) captures the associated invariance with a metric (gray ellipses).}
    \label{fig:relu}
\end{SCfigure}

Thus the approximate posterior cannot reflect the fundamental property of the true posterior, that it should assign a single unique density to a function regardless of its parametrization.

\begin{wrapfigure}[14]{r}{0.5\textwidth}
  \centering
  \vspace{-5mm}
  \includegraphics[width=\linewidth]{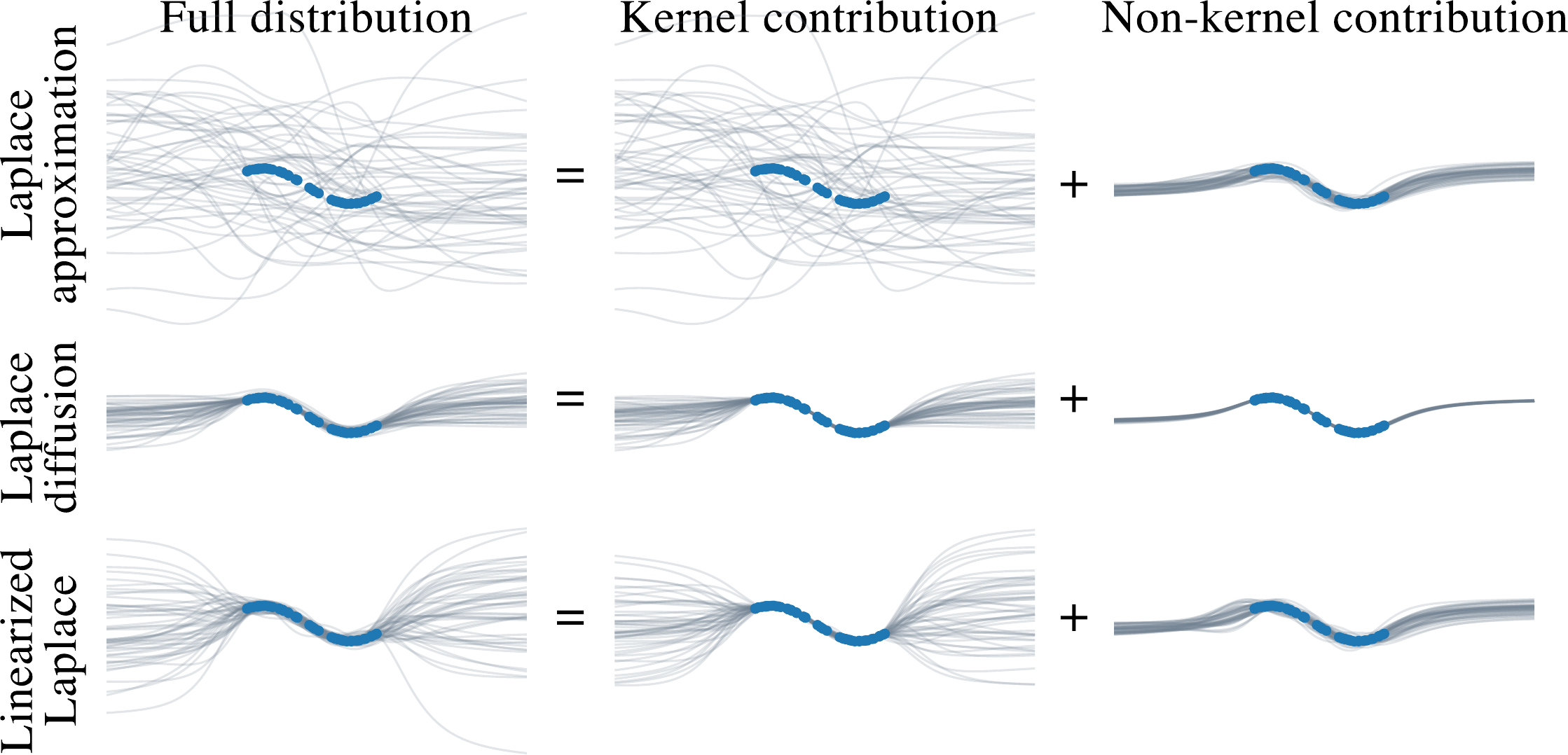}
  \vspace{-5mm}
  \caption{The \emph{function space} is decomposed into directions of \emph{reparameterizations} (kernel) and \emph{functional change} (non-kernel). We improve the posterior fit by concentrating probability mass on directions of functional change.}
  \label{fig:laplace_fails}
\end{wrapfigure}
%
\textbf{In this paper}, we analyze the reparameterization group driving deep learning and show that it is a pseudo-Riemannian manifold, with the \emph{generalized Gauss-Newton (\textsc{ggn})} matrix as its pseudo-metric. 
We prove that the commonly observed underfitting of the Laplace approximation (Fig.~\ref{fig:laplace_fails}, top left) is caused by high in-distribution uncertainty in directions of reparametrizations (Fig.~\ref{fig:laplace_fails}, top center). \looseness=-1

We develop a reparametrization invariant diffusion posterior that proveably does not underfit despite using the neural network predictive (Fig.~\ref{fig:laplace_fails}, center row). Figure~\ref{fig:relu} (right) visualizes how this posterior adapts to the geometry of reparametrizations, thereby not underfitting. The diffusion can be simulated with a multi-step Euler-Maruyama scheme from which the linearized Laplace approximation (\textsc{lla}) is a single-step. 
This link implies that the \textsc{lla} \emph{infinitesimally} is invariant to reparameterizations, due to the otherwise counterintuitive linearization (Fig.~\ref{fig:relu}, left).
Experimentally, our diffusion consistently improves posterior fit, suggesting that reparameterizations should be given more attention in Bayesian deep learning. \looseness=-1

\section{Background: Laplace approximations}
\label{2}
Let $f_{\w}: \mathbb{R}^I \rightarrow \mathbb{R}^O$ denote a neural network with weights $\w$, and define a likelihood $p(\y | f_{\w}(\x))$ and a prior $p(\w)$. \emph{Laplace's approximation} \citep{mackay1992laplace} performs a second-order Taylor expansion of the log-posterior around a mode $\wh$. This results in a Gaussian approximate posterior $\mathcal{N}(\w | \wh, -\mat{H}_{\wh}^{-1})$, where $\mat{H}_{\w}$ is the Hessian matrix.
\emph{The linearized Laplace approximation} \citep{immer2021improving, khan2019approximate} further linearize $f_{\w}$ at a chosen weight $\wh$, i.e.\@ $f_{\w}(\x) \approx f^{\wh}_{\mathrm{lin}}(\w, \x)) = f_{\wh}(\x) +\J_{\wh}(\x)(\w - \wh)$, where $\J_{\wh}(\x) = \partial_{\w} f_{\w}(\x)|_{\w = \wh} \in \R^{O \times D}$ is the Jacobian of $f_{\w}$. Here $D = \mathrm{dim}(\w)$ denotes the number of parameters in the network. Applying the usual Laplace approximation to the linearized model yields an approximate posterior \citep{immer2021improving},\looseness=-1
\begin{align}
  q(\w | \mathcal{D}) = \N\left(\w ~|~ \wh, (\ggn_{\wh} + \alpha\mat{I})^{-1}\right) \qquad 
  \ggn_{\wh} = \sum_{n=1}^N \J_{\wh}(\x_n)\T \Hess(\x_n) \J_{\wh}(\x_n), 
\end{align}
where $\Hess(\x) = -\partial^2_{f_{\wh}(\x)} \log p(\y | f_{\wh}(\x)) \in \R^{O \times O}$ is the Hessian of the log-likelihood and we have assumed a weight prior $\mathcal{N}(\vec{0}, \alpha^{-1} \mat{I})$. Note that it is trivial to extend to other prior covariances.
This particular covariance is known as the \emph{generalized Gauss-Newton} (\ggn) Hessian approximation, which is commonly used in Laplace approximations \citep{laplace2021}. 

To reduce the notational load we stack the per-datum Jacobians into $\J_{\wh} = [\J_{\wh}(\x_1); \ldots; \J_{\wh}(\x_N)] \in \R^{NO \times D}$ and similarly for the Hessians, and write the \ggn{} matrix as
$\ggn_{\wh} = \J_{\wh}\T \Hess \J_{\wh}$.
For Gaussian likelihoods, the Hessian is an identity matrix and can be disregarded, and for other likelihoods simple expressions are generally available \citep{immer2021improving}. 
\looseness=-1



\paragraph{Sampled and linearized Laplace.}
The Laplace approximation gives a Gaussian distribution $q(\w | \mathcal{D})$ over the weight space with mean $\wh$ and covariance $\Sigma$. A predictive distribution is obtained by integrating the approximate posterior against the model likelihood, 
\begin{align}
    p(\y^*|\x^*\!, \mathcal{D}) = \mathbb{E}_{\w\sim q}[p(\y^*|f(\w, \x^*))] 
                            \approx \frac{1}{S}\! \sum_{i=1}^S p(\y^*|f(\w_i, \x^*)), \quad \w_i \sim q.
\end{align}
We refer to this predictive method as \emph{sampled Laplace}. Recent works have suggested \emph{linearizing} the neural network in the likelihood model to obtain the predictive distribution \citep{immer2021improving},
\begin{align}
    p(\y^*|\x^*\!, \mathcal{D}) = \mathbb{E}_{\w\sim q}[p(\y^*|f^{\wh}_{\mathrm{lin}}(\w, \x^*))] 
                            \approx \frac{1}{S}\! \sum_{i=1}^S p(\y^*|f^{\wh}_{\mathrm{lin}}(\w_i, \x^*)), \quad \w_i \sim q.
\end{align}
This is referred to as \emph{linearised Laplace}. \citet{immer2021improving} argues that the common choice of approximating the posterior precision with the \ggn{} implicitly linearizes the neural network and hence the predictive distribution should be modified for consistency.

Sampled Laplace is known to severely \emph{underfit}, whereas the linearized Laplace approximation does not (\citealt{immer2021improving}; Fig.~\ref{fig:laplace_fails}). It is an open problem \emph{why} the crude linearization is beneficial \citep{papamarkou2024position}. This paper shows that the benefit is linked to the lack of \emph{reparameterization invariance}.\looseness=-1

\textbf{The lack of reparameterization invariance} leads to an additional problem for Laplace approximations. The precision of the approximate posterior is given either by the Hessian or the \ggn. As shown by \citet{dinh2017sharp}, the Hessian of the loss is not invariant to reparameterizations of the neural network, and the same holds for the $\ggn$. Depending on which parametrization of the posterior mode is chosen by the optimizer, we, thus, get different covariances for the approximate posterior. Empirically, this can render Laplace's approximation unstable \citep{warburg:metric:2023}. Figure~\ref{fig:relu} (center) illustrates the phenomena. 

\section{Reparameterizations of linear functions}
\label{reparam}
Deep learning models excel when they are highly \emph{overparametrized}, i.e.\@ when they have significantly more parameters than observations ($D \gg NO$). This introduces many degrees of freedom to the model, which will be reflected in the Bayesian posterior. 
However, as we have argued, traditional approximate Bayesian inference does not correctly capture this and assigns different probability measures to identical functions. Next, we characterize these degrees of freedom to design suitable approximate posteriors. 
To develop the theory, we first consider the linear setting and then extend it to the general case.


\textbf{The reparameterizations of linear functions} can be characterized exactly. Consider $f(\w) = \mat{A}\w + \mat{b}$ and a possible reparameterization, $g:\mathbb{R}^D\rightarrow\mathbb{R}^D$, of this function such that $f(g(\w)) = f(\w)$. It is then evident that $\mat{A}(g(\w) - \w) = \vec{0}$. This implies that for any reparameterization of a linear function, we have $g(\w) - \w \in \Ker{\mat{A}}$, where $\Ker{\mat{A}}$ denotes the \emph{kernel} (nullspace) of $\mat{A}$. Hence, the linear function cannot be reparametrized if we restrict ourselves to the non-kernel subspace of the input space or if $\mat{A}$ has a trivial kernel.\looseness=-1

\textbf{A linearized neural network} $f^{\w'}_{\mathrm{lin}} : \w,\x \mapsto f_{\w'}(\x) +\J_{\w'}(\x)(\w - \w')$ is a linear function in the parameters, where we have linearized around $\w'$. The above analysis then implies that the kernel of the stacked Jacobian $\J_{\w'}$ characterizes the reparameterizations of the linearized network.\looseness=-1

We can also characterize the reparameterizations through the \ggn{} and the corresponding \emph{neural tangent kernel} (\textsc{ntk}; \citealt{jacot2018neural}),
\begin{equation}
    \ggn_\w = \J_\w^\top \J_\w,
    \qquad
    \ntk_\w = \J_\w \J_\w^\top.
\end{equation}
By construction, these have the same non-zero eigenvalues, and thereby also have identical ranks. We, thus, see that the kernel of the Jacobian coincides with that of the \ggn{}, i.e.\@ $\Ker{\J_\w} = \Ker{\ggn_\w}$.

\begin{wrapfigure}[15]{r}{0.5\textwidth}
  \centering
  \vspace{-6mm}
  \includegraphics[width=0.9\linewidth]{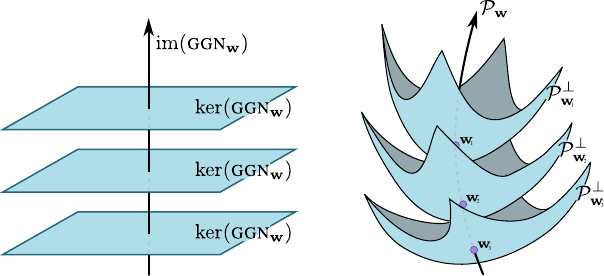}
  \caption{The weight space can be decomposed into directions of \emph{reparameterizations} and \emph{functional changes}. For linear models (left) these are linear subspaces given by the kernel and the image, respectively. For nonlinear models, these are the nonlinear manifolds $\mathcal{P}_{\w_i}^{\perp}$ and $\mathcal{P}_{\w_i}$, respectively.\looseness=-1}
  \label{fig:stacked_kernels}
\end{wrapfigure}
\paragraph{Two orthogonal subspaces.}
For any self-adjoint operator (such as positive semi-definite matrices like the {\ggn}), the \emph{image} and the \emph{kernel} orthogonally span the whole space, i.e.
\begin{equation}\label{eq:orthogonal_decomposition_linear_case}
    \im{\ggn_\w}
    \oplus
    \Ker{\ggn_\w}
    =
    \mathbb{R}^D,
\end{equation}
where the \emph{kernel} is the hyperplane of vectors that are mapped to zero and the \emph{image} is the hyperplane of vectors spanned by the operator (Fig.~\ref{fig:stacked_kernels}).
For a linearized neural network, $\im{\ggn_\w}$ spans the \emph{effective} parameters $\mathcal{P}\subset\R^D\!$, i.e.\@ the maximal set of parameters that generate different linear functions $\R^I\rightarrow\R^O$ when evaluated on the training set. \looseness=-1 

\textbf{A Laplace covariance decomposes} into the same subspaces. 
Recall that the posterior precision is $ \Sigma^{-1} = \ggn_{\wh} + \alpha \mat{I}$. Let the eigendecomposition of $\ggn_{\wh}$ be $\mat{U}^T \mat{\Lambda} \mat{U}$, and assume that $\mat{U_1}$ and $\mat{U_2}$ are the eigenvectors corresponding to the non-zero eigenvalues $\tilde{\mat{\Lambda}}$, and the zero eigenvalues respectively. These form a basis in the kernel and image subspace as discussed above. Then the covariance is,\looseness=-1
\begin{align}
\Sigma =  \left( \left[ \begin{array}{c}
            \mat{U_1} \\
            \hline
            \mat{U_2} 
            \end{array} \right]^T
            \left[ \begin{array}{c | c}
            \tilde{\mat{\Lambda}} & \mat{0}\\
            \hline
            \mat{0} & \mat{0}
            \end{array}  \right]
            \left[ \begin{array}{c}
            \mat{U_1} \\
            \hline
            \mat{U_2} 
            \end{array} \right] + \alpha \mat{I} \right)^{-1} 
       = \mat{U^{\mathrm{T}}_1} (\tilde{\mat{\Lambda}} + \alpha \mat{I}_k)^{-1} \mat{U_1} + \alpha^{-1} \mat{U^{\mathrm{T}}_2}  \mat{U_2}.
\end{align}
Consequently, we can decompose any sample from the Gaussian $\mathcal{N}(\wh, \Sigma)$ into a kernel and an image contribution,
    $\w = \wh + \w_{\mathrm{ker}} + \w_{\mathrm{im}}$,
%
where $\w_{\mathrm{ker}}$ is the component of the sample that is in the kernel of $\ggn_{\wh}$ and $\w_{\mathrm{im}}$ is in the image.
Note that all probability mass in $\Ker{\ggn_{\wh}}$ is due to the prior, i.e.\@ we place prior probability on functional reparameterizations even if we can never observe data in support of such. 

\begin{wrapfigure}[15]{r}{0.5\textwidth}
  \centering
  \vspace{-10mm}
  \includegraphics[width=\linewidth]{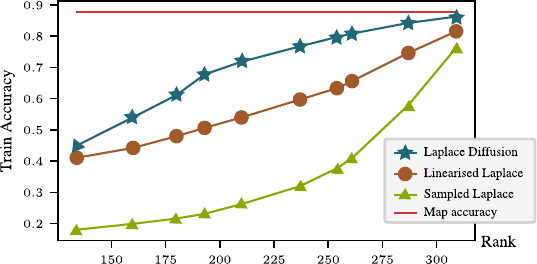}
  \vspace{-4mm}
  \caption{Underfitting of sampled Laplace is less pronounced when the rank of the \ggn{} is higher for a fixed number of parameters. This is consistent with our hypothesis as a high \ggn{} rank implies a lower dimensional kernel. For experimental details, see appendix~\ref{sec: toy_results}.}
  \label{fig:rank_plot}
\end{wrapfigure}
\textbf{Underfitting in sampled Laplace} can now be understood.
For the linearized approximation, it holds for training data $\x\in\mathcal{X}$ that, 
\begin{equation*}
    f^{\wh}_{\mathrm{lin}}(\wh \!+\! \w_{\mathrm{ker}} \!+\! \w_{\mathrm{im}}, \x) = f^{\wh}_{\mathrm{lin}}(\wh \!+\! \w_{\mathrm{im}}, \x).
\end{equation*}
Hence, the linearized predictive only samples in the image subspace consisting of unique functions. This is \emph{not} true for the sampled Laplace approximation, which also samples in the kernel subspace. Since sampled Laplace does not linearize the neural network, the kernel does not correspond to reparameterizations. It hence adds ``incorrect'' degrees of freedom to the posterior as artifacts of the Gaussian approximation. 





Empirically, sampled Laplace is only observed to underfit in overparametrized models. Fig.~\ref{fig:rank_plot} illustrates this by increasing the amount of training data to decrease the kernel rank, i.e.\@ reduce the reparametrization issue. We find that as the issue is lessened, sampled Laplace reduces its underfitting.

\section{Reparameterizations of neural networks}
We have seen that the parameters of linear models can be decomposed into two linear subspaces corresponding to reparameterizations and functional changes. We next analyze nonlinear models. 

\textbf{Intuitively}, reparameterizations of a nonlinear neural network form continuous trajectories in the parameter space (c.f.\@ Fig.~\ref{fig:relu}). We define that all points along such a trajectory are identical, which changes the weight space geometry to be a manifold. Likewise, the parameter changes corresponding to actual function changes reside on a nonlinear manifold. This is sketched in Fig.~\ref{fig:stacked_kernels}. Interestingly, the \ggn{} turns out to induce a natural (local) inner product on these nonlinear manifolds, which allows us to both understand and generalize the linearized Laplace approximation. 



\subsection{The effective-parameters quotient space}
For a \emph{nonlinear} neural network $f:\R^D\times\R^I\rightarrow\R^O$, the surfaces in weight space along which the function does not change are generally \emph{not} linear. Here, we formalize these reparameterization invariant surfaces and show that they are a partition of the weight space.
\begin{definition}
    Given a datapoint $\x\in\R^I$, for any $\w\in\R^D$ we define the \emph{$\x$-reparameterizations} as the set
        $\mathcal{R}^f_{\x}(\w)
        =
        \{
        \w' 
        \textnormal{ such that }
        f(\w', \x) = f(\w, \x) 
        \}$. 
    Consistently, given a collection of points $\mathcal{X}\subseteq\mathbb{R}^I$, we call the intersection
    $\mathcal{R}^f_{\mathcal{X}}(\w)
        =
        \bigcap_{\x\in\mathcal{X}}
        \mathcal{R}^f_{\x}(\w)$
    \emph{$\mathcal{X}$-reparameterizations}.
\end{definition}
Trivially, $\w\in\mathcal{R}^f_{\mathcal{X}}(\w)$ for any choice of $\mathcal{X}$. We next define the subset of $\mathcal{X}$-reparameterizations which can be obtained via a smooth deformation from $\w$.
\begin{definition}
    We say that a piecewise differentiable function $\gamma:[0,1]\rightarrow\mathbb{R}^D$ is a \emph{homotopy} of $(\w,\hat{\w})$ if $\gamma(0)=\w$ and $\gamma(1)=\w'$.
    The set of \emph{$\mathcal{X}$-smooth-reparameterizations} is defined as,
    \begin{equation*}
        \bar{\mathcal{R}}^f_{\mathcal{X}}(\w)
        =
        \left\{
        \w' 
        \textnormal{ such that}
        \begin{tabular}{l}
            $\exists \gamma$ a homotopy of $(\w,\w')$  \\
            $\gamma(t)\in\mathcal{R}^f_{\mathcal{X}}(\w)\,\,\forall t\in[0,1]$
        \end{tabular}
        \right\}.
    \end{equation*}
\end{definition}
A homotopy $\gamma$ is, thus, a smooth path along which all neural networks have identical predictions on $\mathcal{X}$. We consider two networks, $\w$ and $\w'$, similar if they can be connected by such a homotopy. Formally, 
we define the relation $\sim$ over $\mathbb{R}^D$ as $\w {\sim} \w'$ if $\w'\in\bar{\mathcal{R}}^f_{\mathcal{X}}(\w)$. 

We next use this relation to form a new view on the weight space $\R^D$ in which similar weights are seen as \emph{one} point. This can be realized using \emph{quotient spaces} \citep{lee2012smooth}. These are well-studied spaces that are constructed by considering a collection of points in one space as a single point in a new space. In our case, we have the following result.
\begin{lemma}\label{thm:quotient_group_existence}
    $\sim$ is an equivalence relation, i.e.\@ it is transitive, symmetric and reflexive. We can form the quotient space
        $\mathcal{P} = \mathbb{R}^D / \sim$
    of effective parameters. 
    We denote $[\w]\in\mathcal{P}$ the equivalence class of an element $\w\in\mathbb{R}^D$.
\end{lemma}
This quotient structure gives a rich mathematical foundation to construct reparameterization invariant neural networks.
Within the quotient, two effective parameters $[\w_1],[\w_2]\in\mathcal{P}$ are the same point if and only if $\w_1 \sim \w_2$. This means that all parameters $\w\in[\w_1]$ gives the same function over $\mathcal{X}$.


\subsection{The effective-parameters manifold}\label{sec:effect_manifold}

Geometry is the mathematical language of invariances. To this end would like to endow the weight space with a geometric structure such that two weights, $\w_1$ and $\w_2$, corresponding to the same function, have a distance of zero, i.e.
\begin{align}
  \mathrm{dist}(\w_1, \w_2) = 0 \quad \Leftrightarrow \quad \w_1 \sim \w_2.
  \label{eq:dist_zero}
\end{align}
Since the weights generate the same function, we define a metric that measures differences in \emph{function values} on the training data. Consider weights $\w$ and an infinitesimal displacement $\bm{\epsilon}$, we then define,
\begin{align}
  \mathrm{dist}^2(\w, \w+\bm{\epsilon})
    = \sum_{n=1}^N \| f(\w, \x_n) - f(\w+\epsilon, \x_n) \|^2 
    = \bm{\epsilon}^\top \ggn_\w \bm{\epsilon} + \mathcal{O}(\epsilon^3),
  \label{eq:dist_def}
\end{align}
where the last step follows from a first-order Taylor expansion of $f$ around $\w$. This is a standard \emph{pullback metric} $(f^*H)_{\w}=\ggn_\w$ commonly used in Riemannian geometry. 
This implies that the \ggn{} matrix infinitesimally defines an inner product, i.e.\@ it is a \emph{Riemannian metric}. 
By integrating over paths, the distance extend to any pair of points and satisfies Eq.~\ref{eq:dist_zero} \citep{lee2012smooth}.



\paragraph{Watch out!~It's a pseudo-metric.} We have already seen that in overparametrized models, the \ggn{} is rank-deficient, which implies that it is not positive definite. Consequently, it is not a Riemannian metric but rather a \emph{pseudo}-Riemannian metric. 
A pseudo-metric can be a counterintuitive object: two points $\w_1$ and $\w_2$ at distance zero may have different pseudo-metrics $(f^*H)_{\w_1} \neq (f^*H)_{\w_2}$. 
This is reflected in the Laplace approximation.
The covariance prescribed by the Laplace approximation is $\Sigma_{\wh}=(\nabla^2_{\w}\mathcal{L}(\wh)+\alpha\identity)^{-1}$, where $\mathcal{L}(\w)$ is shorthand for the training log-likelihood. The Hessian is exactly the pullback pseudo-metric $\nabla^2_{\w}\mathcal{L}(\wh) = (f^*H)_{\wh}$, which is \emph{not} invariant to reparameterizations of the neural network. Specifically, for a reparameterization function $g$ that is also a diffeomorphism, the change of variable rules states that,
\begin{equation}
    \underbrace{
    \nabla^2_{\w}\mathcal{L}(g(\wh))
    }_{\Sigma^{-1}_{g(\wh)} - \alpha\identity}
    =
    \nabla_{\w}g(\wh)^\top
    \underbrace{
    \nabla^2_{\w}\mathcal{L}(\wh)
    }_{\Sigma^{-1}_{\wh} - \alpha\identity}
    \nabla_{\w}g(\wh).
\end{equation}
This means that, while each parameter $\w$ has its well-defined covariance $\Sigma_{\w}$, each equivalence class does not have a unique one, since $\wh$ and $g(\wh)$ belong to the same equivalence class and $\Sigma_{\wh}\not=\Sigma_{g(\wh)}$.

\paragraph{Non-Gaussian likelihoods.}
The Euclidean distance measure in Eq.~\ref{eq:dist_def} corresponds to choosing a Gaussian likelihood. The distance definition readily extends to other likelihoods and the corresponding metric takes the form of the generalized Gauss-Newton matrix $\mat{J}_\w^\top \mat{H} \mat{J}_\w$, where $\mat{H}$ denotes the Hessian of the log-likelihood. For both Gaussian and Bernoulli likelihoods, this Hessian is positive definite, but e.g.\@ the cross entropy has a rank-deficient Hessian and, thus, induces a pseudo-metric.


\paragraph{An impractical solution.}
The unfortunate behavior of approximate posteriors assigning different probabilities to the same function could 
be rectified by marginalizing over the set of reparameterizations of $\w$, i.e.
   $\int_{\w' \in \mathcal{R}(\w)} q(\w' | \mathcal{D}) \dif\w'$.
While this construction solves the highlighted problem, its complexity makes it 
impractical and we are unaware of any works along these lines.

When restricted to a smaller class of reparameterization (the ones homotopic to the identity), the integral can be thought of as ``collapsing'' each reparameterization equivalence class to a single point in $\mathcal{P}=\R^D/\sim$ formalized in \cref{thm:quotient_group_existence}. Nontrivially, the pullback metric implicitly performs a similar operation, as shown later in \cref{thm:equivalence_quotient_and_pullback}. This connection motivates the dive into Riemannian geometry: \emph{we get a tractable approach to engaging with neural network reparameterizations}.

\subsection{Topological equivalence of the two views}
So far we described two \emph{a priori} very different objects: the quotient space $\mathcal{P} = \mathbb{R}^D/\sim$ and the pseudo-Riemannian manifold $(\mathbb{R}^D, \ggn_\w)$. We referred to both of them as \emph{effective parameters} and this is no coincidence as there is a natural relationship between the points at distance zero according to the pseudo-metric and the equivalence classes.
\begin{proposition}\label{prop:equivalence_quotient_and_pullback}
    For any $\w_0,\w_1\in\R^D$ it holds
    \begin{equation}
        d_{f^*H}(\w_0,\w_1) = 0
        \quad\Longleftrightarrow\quad
        [\w_0] = [\w_1] \in \mathcal{P}.
    \end{equation}
\end{proposition}
Even better, these two spaces share the same topological structure.
To state this we need a notion of distance on the quotient space and the most natural choice is to inherit the Euclidean distance $\|\cdot\|$ from $\mathbb{R}^D$. This distance is defined as
\begin{equation*}
    d_\mathcal{P}([\w],[\w'])
    =
    \inf
    \left\{
        \|p_1-q_1\| + \ldots + \|p_n-q_n\|
    \right\},
\end{equation*}
where the infimum is taken over all finite sequences $p_1,\ldots,p_n$ and $q_1,\ldots,q_n$ such that $[\w]=[p_1]$, $[p_{i+1}]=[q_i]$ and $[q_n]=[\w']$. 

This distance $d_\mathcal{P}$ induces a topology on the quotient space $\mathcal{P}$ which is equivalent to the topology induced by the pullback distance $d_{f^*H}$ on the pseudo-Riemannian manifold. Formally
\begin{theorem}\label{thm:equivalence_quotient_and_pullback}
    For any $\w_0,\w_1\in\R^D$, for any $\epsilon>0$ there exists $\delta>0$ such that
    \begin{align}
        d_{\mathcal{P}}([\w_0],[\w_1]) < \delta
        & \quad\Longrightarrow\quad
        d_{f^*H}(\w_0,\w_1) < \epsilon\\
        d_{f^*H}(\w_0,\w_1) < \delta
        & \quad\Longrightarrow\quad
        d_{\mathcal{P}}([\w_0],[\w_1]) < \epsilon.
    \end{align}
\end{theorem}

This result connects an abstract quotient space $\mathcal{P}$ with the pseudo-Riemannian metric $\ggn_\w$. The quotient captures useful intuitions but is difficult to leverage computationally. In contrast, the pseudo-metric has some counterintuitive aspects but we can identify the underlying Riemannian structure which leads to tractable algorithms (Sec.~\ref{sec:diffusion_posteriors}).
%
%
%


\paragraph{A tale of two manifolds.}
For any given parameter $\w\in\mathbb{R}^D$ and training set $\mathcal{X}$, we show that there exist two Riemannian manifolds $(\mathcal{P}_\w,\m)$ and $(\mathcal{P}_\w^\perp,\m^{\perp})$ embedded in $\mathbb{R}^D$, illustrated in \cref{fig:stacked_kernels}. They capture the functional change and reparameterization properties respectively, but, differently from the previously studied $(\R^D,\ggn_\w)$, they are Riemannian manifolds without degenerate directions in their metrics. Formally,

\begin{theorem}
    For any parameter $\w$ suppose the set of parameters that generate the same predictions is denoted by $\mathcal{P}_\w^\perp = \{ \w' \in \mathbb{R}^D \textnormal{ such that } f(\w',x) = f(\w,x) \textnormal{ for all }x \in \mathcal{X} \}$. Then this set is a smooth manifold embedded in $\mathbb{R}^D$. Furthermore, the set of parameters that locally generates unique predictions, $\mathcal{P}_\w$ is also a submanifold embedded in $\mathbb{R}^D$. 
\end{theorem}
%
They are the direct generalization to the nonlinear case of the two spaces involved in Eq.~\ref{eq:orthogonal_decomposition_linear_case}, where $\mathcal{P}_\w$ plays the role of the \emph{image} and $\mathcal{P}_\w^\perp$ plays the role of the \emph{kernel}. When $f$ is linear, they are identical.

In general, $\mathcal{P}_\w$ and $\mathcal{P}_\w^\perp$ intersect only in $\w$, 
and the two respective tangent spaces in $\w$ span all directions.
They can be thought of as two collections of parameters, and the associated functions have different properties: (1) $\mathcal{P}_\w^\perp$ is entirely contained in the same equivalence class $\mathcal{P}_{\w}^\perp\subseteq[\w]$, thus all the parametrized functions are identical on the train set; in contrast, (2) $\mathcal{P}_\w$ never intersects the same equivalence class more than one time, at least locally, thus the parametrized functions always changes when moving in any direction. Thus $\mathcal{P}_\w$ resembles the effective parameter manifold $\mathcal{P}$, but with the difference of being an actual Riemannian manifold.
These two manifolds exist under the assumption that Jacobian is full rank (see proof in appendix~\ref{sec: C}).

The two metrics $\m$ and $\m^{\perp}$ are not uniquely defined. A natural choice for $\m$ is to restrict $\ggn_\w$ to the tangent space of $\mathcal{P}_\w$ corresponding to the non-zero eigenvectors, i.e.\@ $\m=\ggn_\w^+$. While, for $\m^{\perp}$ we can inherit the Euclidean metric, i.e.\@ $\m^{\perp}=\alpha\mat{I}$ for $\alpha>0$. 

\section{Exploring manifolds with random walks}\label{sec:diffusion_posteriors}




\paragraph{SDEs on manifolds.}
Given a Riemannian manifold $(\M, \mat{G})$, the simplest choice of distribution that respects the Riemannian metric $\mat{G}$ is a Riemannian diffusion (or Brownian motion, c.f.\@ \citet{hsu2002stochastic}) stopped at time $t$.
This follows the stochastic differential equation \citep{girolami2011riemann},
 \begin{align}
      \dif\w = \sqrt{2\tau}\mat{G}(\w)^{-\frac{1}{2}} \dif W + \tau \Gamma \dif t 
      \qquad \textnormal{where} \quad \Gamma_i(\w) = \sum_{j=1}^D \frac{\partial}{\partial \w_j} (\mat{G}(\w)^{-1})_{ij}. 
 \end{align}
Practically speaking this simple process can be simulated using an Euler–Maruyama \citep{maruyama1955continuous} scheme. The Christoffel symbols, $\Gamma_i(\bm{\theta})$, are commonly disregarded as they have a high computational cost, and \citet{li2015preconditioned} showed that the resulting error is bounded.

Using the Euler–Maruyama integrator with step size $h_t$, setting $\tau=1$ corresponding to standard Bayesian inference and disregarding the term involving the Christoffel symbols $\Gamma$, we obtain the simple update rule 
    $\w_{t+1} = \w_t + \sqrt{2h_t}\mat{G}(\w_t)^{-\frac{1}{2}} \bm{\epsilon}$,
where $\bm{\epsilon} \sim \mathcal{N}(\vec{0}, \mat{I})$. 
%
%
 %
%
%
This applies to any Riemannian manifold. However, the effective-parameter $(\mathbb{R}^D,\ggn_\w)$ is only pseudo-Riemannian, we explore three Riemannian alternatives: $(\mathbb{R}^D,\ggn_\w + \alpha\mathbb{I})$, $(\mathcal{P}_\w^\perp,\alpha\mat{I})$ and $(\mathcal{P}_\w,\ggn_{\bm{\omega}}^+)$

\paragraph{Diffusion on $(\mathbb{R}^D,\ggn_{\w} + \alpha\mat{I})$.} 
The Laplace approximation can also be written as a diffusion on a manifold. As we saw in Sec.~\ref{2}, the Laplace approximation can be written as
    $\w|\mathcal{D} \sim \mathcal{N}(\wh, \Sigma)$ with $\Sigma^{-1} = \ggn_{\wh} + \alpha\mat{I}$.
This can also be written as a sample at $t=1$ of a Riemannian diffusion on a manifold with a constant metric, $(\mathbb{R}^D, \mat{G})$, where $\mat{G} = \ggn_{\wh} + \alpha\mat{I}$. 
The \textsc{sde} $\dif\w = \mat{G}^{-\frac{1}{2}} \dif W$ have a marginal distribution at $t=1$ that exactly match the standard Laplace approximation. 
Note that this formulation does not rely on the approximation of the \textsc{sde} that disregards the term involving the Christoffel symbols $\Gamma$ as these are zero for constant metrics. Hence, the above is exactly a Riemannian diffusion on the manifold with a constant metric given by the \ggn{} at the \textsc{map} parameter. Note that this is only a valid diffusion for $\alpha > 0$ in which case it is not reparametrization invariant.


\begin{table}
  \caption{In-distribution performance across methods trained on \textsc{mnist}, \textsc{fmnist} and \smaller[0.80]{CIFAR-10}.}
  \label{tab:id}
  %
  \setlength{\aboverulesep}{0pt}
  \setlength{\belowrulesep}{0pt}
  \setlength{\extrarowheight}{.75ex}
  \resizebox{\textwidth}{!}{
    \rotatebox[origin=c]{90}{MNIST~~~~~}\hspace{3mm}
    \begin{tabular}{lcccccc}
    \toprule \rowcolor{gray!30}
                            & Conf.~($\uparrow$)           & NLL~($\downarrow$)          & Acc.~($\uparrow$)           & Brier~($\downarrow$)        & ECE~($\downarrow$)        & MCE~($\downarrow$) \\
    \midrule
    Laplace Diffusion (ours) & \first{0.988\small{±0.001}} & \first{0.042\small{±0.007}} & \first{0.987\small{±0.002}} & \first{0.022\small{±0.003}} & \first{0.137\small{±0.019}} & \first{0.775\small{±0.043}} \\
    Sampled Laplace          & 0.589\small{±0.008}         & 3.812\small{±0.284}         & 0.146\small{±0.032}         & 1.176\small{±0.046}         & 0.443\small{±0.026}        & 0.985\small{±0.002} \\
    Linearised Laplace       & 0.968\small{±0.004}         & 0.306\small{±0.041}         & 0.926\small{±0.008}         & 0.117\small{±0.012}         & 0.251\small{±0.034}        & 0.855\small{±0.041} \\
    \bottomrule
  \end{tabular} }
    %
  \resizebox{\textwidth}{!}{
    \rotatebox[origin=c]{90}{FMNIST}\hspace{3mm}
    \begin{tabular}{lcccccc}
    Laplace Diffusion (ours) & \first{0.900\small{±0.001}} & \first{0.001\small{±0.000}} & \first{0.906\small{±0.007}} & \first{0.141\small{±0.006}} & \first{0.108\small{±0.015}} & \first{0.729\small{±0.092}} \\
    Sampled Laplace          & 0.618\small{±0.021}         & 4.507\small{±0.000}         & 0.098\small{±0.010}         & 1.295\small{±0.014}         & 0.518\small{±0.013}         & 0.986\small{±0.001} \\
    Linearised Laplace       & 0.897\small{±0.003}         & 0.423\small{±0.000}         & 0.862\small{±0.005}         & 0.207\small{±0.006}         & 0.147\small{±0.017}         & 0.756\small{±0.048} \\
    \bottomrule
  \end{tabular} }
    %
  \resizebox{\textwidth}{!}{
    \rotatebox[origin=c]{90}{CIFAR-10}\hspace{3mm}
    \begin{tabular}{lcccccc}
    Laplace Diffusion (ours) & \first{0.952\small{±0.007}} & \first{0.345\small{±0.062}} & \first{0.905\small{±0.007}} & \first{0.155\small{±0.019}} & 0.259\small{±0.008}         & 0.870\small{±0.021} \\
    Sampled Laplace          & 0.843\small{±0.004}         & 0.997\small{±0.222}         & 0.717\small{±0.049}         & 0.422\small{±0.081}         & \first{0.221\small{±0.047}} & 0.804\small{±0.080} \\
    Linearised Laplace       & 0.951\small{±0.007}         & 0.614\small{±0.020}         & 0.863\small{±0.001}         & 0.222\small{±0.002}         & 0.337\small{±0.022}         & \first{0.789\small{±0.035}} \\
    \bottomrule
  \end{tabular} }
\end{table}

\begin{table}
  \caption{Out-of-distribution \textsc{auroc}~($\uparrow$) performance for \textsc{mnist}, \textsc{fmnist} and \smaller[0.80]{CIFAR-10}.}
  \label{tab:ood}
  %
  \setlength{\aboverulesep}{0pt}
  \setlength{\belowrulesep}{0pt}
  \setlength{\extrarowheight}{.75ex}
  %
  \resizebox{\textwidth}{!}{
    \begin{tabular}{lcccccccc}
    \toprule 
    \rowcolor{gray!30} Trained on & \multicolumn{3}{c}{\rule[0.9mm]{20mm}{0.1mm} \textsc{mnist} \rule[0.9mm]{20mm}{0.1mm}} & \multicolumn{3}{c}{\rule[0.9mm]{20mm}{0.1mm} \textsc{fmnist} \rule[0.9mm]{20mm}{0.1mm}} & \multicolumn{2}{c}{\rule[0.9mm]{12mm}{0.1mm} \smaller[0.80]{CIFAR-10} \rule[0.9mm]{12mm}{0.1mm}} \\
    \rowcolor{gray!30} Tested on & \textsc{fmnist} & \textsc{emnist} & \textsc{kmnist} & \textsc{mnist} & \textsc{emnist} & \textsc{kmnist} & \smaller[0.80]{CIFAR-100} & \textsc{svhn} \\
    \midrule
    Laplace Diffusion (ours) & \first{0.909\small{±0.033}} & \first{0.625\small{±0.018}} & \first{0.929\small{±0.008}} & \first{0.759\small{±0.045}} & \first{0.741\small{±0.010}} & \first{0.749\small{±0.023}} & \first{0.851\small{±0.002}} & \first{0.862\small{±0.010}} \\
    Sampled Laplace          & 0.500\small{±0.026}         & 0.494\small{±0.006}         & 0.482\small{±0.013}         & 0.495\small{±0.037}         & 0.503\small{±0.036}         & 0.493\small{±0.033}         & 0.687\small{±0.033}         & 0.599\small{±0.038} \\
    Linearised Laplace       & 0.758\small{±0.070}         & 0.602\small{±0.027}         & 0.790\small{±0.018}         & 0.625\small{±0.050}         & 0.628\small{±0.013}         & 0.624\small{±0.020}         & 0.837\small{±0.006}         & 0.854\small{±0.024} \\
    \bottomrule
  \end{tabular} }
\end{table}

\paragraph{Kernel-manifold diffusion.}
The kernel-manifold $(\mathcal{P}_\w^\perp,\alpha\mat{I})$ consists of parameters that generate the same function over the training set. The effect of diffusion on this manifold and using the neural network predictive is similar to sampling from the kernel subspace while using the linearized predictive. On the training set the predictive variance is $0$ because it only samples reparametrizations of the \textsc{map} predictions. On out-of-distribution data, the variance is greater than $0$ if at least one of the reparameterizations on the training set is not a global reparameterization. This leads to a clear separation in the predictive variance of in-distribution and out-of-distribution data (Fig.~\ref{fig:laplace_fails}) and further implies that this diffusion distribution never underfits. Stated formally,
%
%
%
\begin{theorem}
\label{thm:5.1}
     $ \mathrm{Var}_{\w\sim \mathcal{P}_{\wh}^\perp}\left[ f(\w,\x) \right] = 0$ for train data $\x\in\mathcal{X}$.
    For a test point $\x_{t}\not\in\mathcal{X}$, if there exists a reparameterization $\w'\in\bar{\mathcal{R}}^f_{\mathcal{X}}(\wh)$ such that $\w'\not\in\bar{\mathcal{R}}^f_{\mathcal{X}\cup\{\x_{t}\}}\!(\wh)$, then
     $   \mathrm{Var}_{\w\sim\mathcal{P}_{\wh}^\perp}\left[ f(\w,\x_{t}) \right] > 0 $.
\end{theorem}

\paragraph{Non-kernel-parameter manifold diffusion.}
The non-kernel-parameter manifold $(\mathcal{P}_\w,\ggn_{\bm{\omega}}^+)$ consists of parameters that generate unique functions over the training set. Diffusion on this manifold samples functions that are necessarily different from the \textsc{map} predictions on the training set. However, the predictive variance in the training set is bounded such that the functional diversity in the predictive samples reflects the intrinsic variance of the training data (Fig.~\ref{fig:laplace_fails}). 

This is the only considered diffusion that acts on a Riemannian manifold while being reparametrization invariant, i.e.\@ $\bar{\mathcal{R}}^f_{\mathcal{X}}(\w)=\{\w\}$. We call this \emph{Laplace diffusion} and study it empirically in Sec.~\ref{sec:experiments}.

\section{Related work}
Bayesian deep learning techniques are still in their infancy and generally involve poorly understood approximations. The arguably most popular tool for uncertainty quantification is \emph{ensembles} \citep{lakshminarayanan2017simple, hansen1990neural}. Several approaches make Gaussian approximations to the true posterior, including \emph{`Bayes by backprop'} \citep{blundell2015weight}, \emph{stochastic weight averaging} (\textsc{swag}) \citep{maddox2019simple} and the \emph{Laplace approximation} \citep{mackay1992laplace, daxberger2022laplace, antorán2023samplingbased, deng2022accelerated, miani:neurips:2022}. 

The high dimensionality of the weight space gives rise to significant computational challenges when constructing Bayesian approximations. This has motivated various low-rank approximations \citep[review in][]{NEURIPS2021_a7c95857}, e.g.\@ \emph{last layer} approximations \citep{pmlr-v119-kristiadi20a}, \emph{subnetwork inference} \citep{daxberger2021bayesian}, \emph{subspace inference} \citep{izmailov2020subspace} or even \textsc{pca} in weight space \citep{maddox2019simple}. Such approaches lessen the computational load, while often improving predictive performance. Our analysis sheds light on why crude approximations perform favorably: \emph{smaller models are less affected by reparameterization issues}. Our diffusion process, thus, provides an alternative, and less heuristic, path forward.\looseness=-1

\citet{mackay1998choice} noted the importance of the choice of basis in Laplace approximations; our pseudo-Riemannian view can be seen as having a continuously changing basis.
\citet{kristiadi2023geometry} studied how a metric transforms under a bijective differentiable change of variables. 
They enforce geometric consistency, highlighting, e.g., the non-invariance of the \ggn{} to a change of variables. \citet{petzka2019reparameterization, jang2022reparametrization} point to the same inconsistency with an emphasis on flatness measures. 

\citet{kim2022scale} and \citet{antoran2022adapting} study global (rather than data-dependant) reparametrizations associated with specialized architectures. While analytic expressions can be obtained, the results do not apply to general networks.
While not expressed in terms of reparametrizations, \citet{izmailov2021dangers} show that linearly dependent datasets give rise to a hyperplane in the kernel manifold. \citet{kim2024gex} also study the kernel of the $\ggn$ in the context of influence functions. These works characterize subsets of the reparametrization group. We provide the first architecture-agnostic characterization of all continuous reparametrizations. 

In a closely related work, \citet{bergamin2024riemannian} introduced a Riemannian Laplace approximation \citep{hauberg:SN:2018} that improves posterior fit over a range of tasks. Furthering this line of research, \citet{yu2023riemannian} explored the use of the Fisher information metric within this framework. While sharing the language of Riemannian geometry, our work focuses on analyzing the effectiveness of linearized Laplace within the context of neural network reparametrization, instead of primarily aiming to achieve better posterior approximations. This allows us to gain deeper insights into the underlying mechanisms that contribute to the success of this approximation technique. \looseness=-1


\begin{SCfigure}[0.7][t]
  \includegraphics[width=0.55\linewidth]{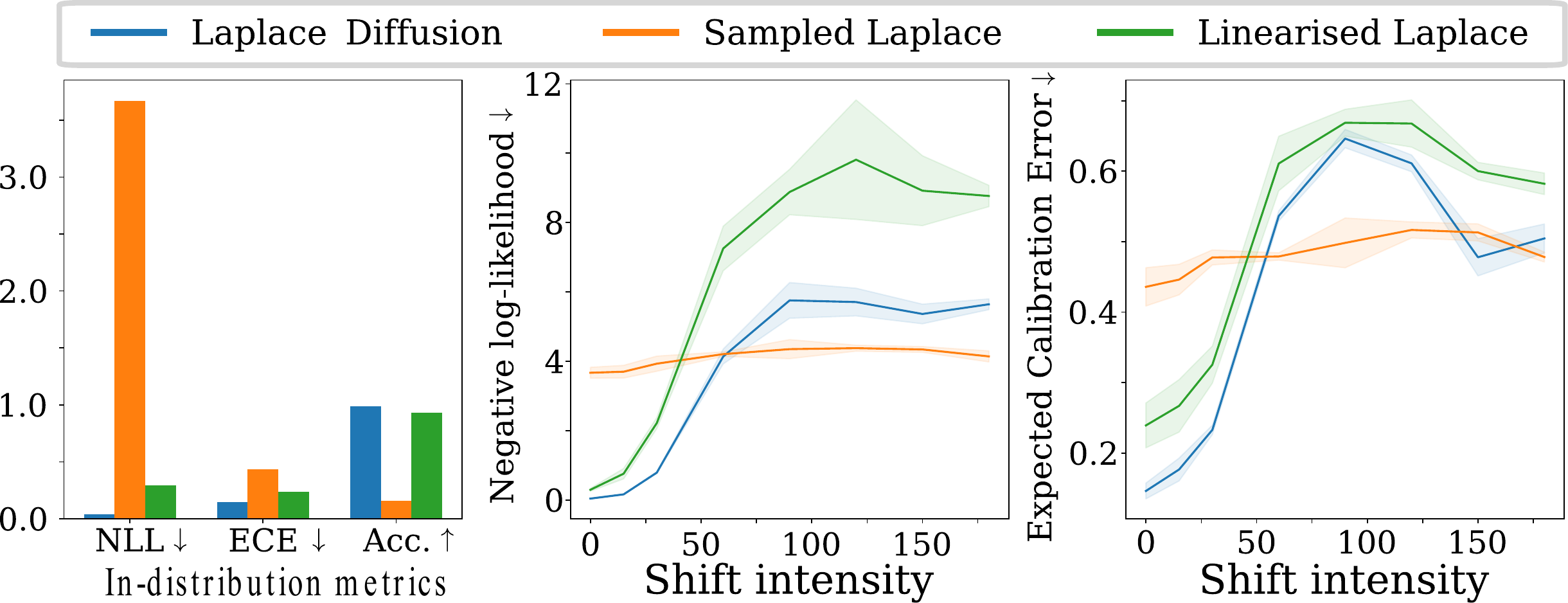}
  \caption{Benchmark results for Rotated \textsc{mnist} (similar results for \textsc{fmnist} and \textsc{cifar} are in appendix~\ref{sec:robustness}). Sampled Laplace significantly underfits even for non-rotated data. Laplace diffusion consistently outperforms the other methods.}
  \label{fig:rotation_plots}
\end{SCfigure}

\section{Experiments}\label{sec:experiments}

\renewcommand{\thefootnote}{}
\footnotetext{Code: {\texttt{https://github.com/h-roy/geometric-laplace}}.}
\vspace{-3mm}

We benchmark Laplace diffusion with neural network predictive against linearized and sampled Laplace to validate the developed theory. Implementation details are in Appendix~\ref{sec: numerics}. We will show that the diffusion posterior slightly outperforms linearized Laplace in terms of both in-distribution fit and out-of-distribution detection. For completeness, we include comparisons to other baselines such as SWAG, diagonal Laplace, and last-layer Laplace in Appendix~\ref{sec: more baselines}. Laplace diffusion is competitive with the best-performing Bayesian methods despite using the neural network predictive (i.e.\@ no linearization). This contrasts sampled Laplace which severely underfits. 
This is evidence that the developed theory explains the key challenges of Bayesian deep learning.


\paragraph{Experimental details (appendix~\ref{sec: more baselines}).}
We train a  44,000-parameter LeNet\citep{lecun1989backpropagation} on \textsc{mnist} and \textsc{fmnist} as well as a 270,000-parameter ResNet\citep{he2016deep} on \textsc{cifar-10}\citep{krizhevsky2009learning}. We sample from the Laplace approximation of the posterior and our Laplace diffusion. For the samples from the Laplace approximation, we consider both the linearized predictive and the neural network predictive, while for diffusion samples, we only consider the neural network predictive. These baselines were chosen to be as similar as possible to our approach to ease the comparison. We use the same prior precision for all methods to ensure a fair comparison. 

\paragraph{In-distribution performance (Table~\ref{tab:id}).}
We measure the in-distribution performance of different posteriors on a held-out test set.  We report means $\pm$ standard deviations of several metrics: Confidence, Accuracy, Negative Log-Likelihood, Brier Score \citep{VERIFICATIONOFFORECASTSEXPRESSEDINTERMSOFPROBABILITY}, Expected Calibration Error \citep{naeini2015obtaining} and Mean Calibration Error. We observe that Laplace diffusion has the best calibration and fit. We also confirm the underfitting of Sampled Laplace across cases. For \textsc{cifar-10} we had to use a large prior precision to get meaningful samples from sampled Laplace, which explains the less severe underfitting. High prior precision is known to help with underfitting in sampled Laplace, but it also shrinks predictive uncertainty to almost zero. 
\paragraph{Robustness to dataset shift (Fig.~\ref{fig:rotation_plots}, appendix~\ref{sec:robustness}).}
We use \textsc{rotated-MNIST}, \textsc{rotated-fmnist}, and \textsc{rotated-cifar} to asses model-calibration and model fit under distribution shift. Fig.~\ref{fig:rotation_plots} plots negative log-likelihood (\textsc{nll}) and expected calibration error (\textsc{ece}) against the degrees of rotation. Laplace diffusion improves on other Laplace approximations.\looseness=-1

\paragraph{Out-of-distribution detection (Table~\ref{tab:ood}).}
On out-of-distribution data from other benchmarks, we see that Laplace diffusion outperforms the other Laplace approximations.

\section{Conclusion}
While approximate Bayesian inference excels in many areas, it continues to face challenges in deep learning. Techniques that work well in shallow models struggle with deep ones even if they remain computationally tractable. This suggests that overparametrization plays a negative role in Bayesian models. Our theoretical analysis shows how overparametrization creates a growing reparameterization issue that conflicts with standard Euclidean approximate posteriors, such as the ever-present Gaussian. For small models this issue is negligible, but as models grow, so does the reparameterization issue.

Our geometric analysis also suggests a solution: we should consider approximate posteriors that respect the group structure of the reparameterizations. We observe that the generalized Gauss-Newton (\ggn) matrix commonly used in Laplace approximations induces a pseudo-Riemannian structure on the parameter space that respects the topology of the reparameterization group. This implies that we can use pseudo-Riemannian probability distributions as approximate posteriors, and we experimented with the obvious choice of a geometric diffusion process. We also showed that the state-of-the-art \emph{linearized} Laplace approximation can be viewed as a naïve (or simple) numerical approximation to our proposed diffusion. This helps explain the success of the linearized approximation.

Our proposed approximate posterior does have issues. While sampling has the same complexity as standard Laplace approximations, it increases runtime by a constant factor. Common Laplace approximations do not sample according to the \ggn{} but rather approximate this matrix with a diagonal or block-diagonal matrix. Mathematically, such approximations break the motivational reparameterization invariance, so it is unclear if such approaches should be applied in our framework. Our work, thus, raises the need for new computational pipelines for engaging with the \ggn{} matrix.

\begin{ack}
This work was supported by a research grant (42062) from VILLUM FONDEN. This project received funding from the European Research Council (ERC) under the European Union’s Horizon 2020 research and innovation programme (grant agreements 757360 and 101123955). The work was partly funded by the Novo Nordisk Foundation through the Center for Basic Machine Learning Research in Life Science (NNF20OC0062606). In addition to the ERC (above), PH, MP and LT thank the International Max Planck Research School for Intelligent Systems (IMPRS-IS) for support, and gratefully acknowledge financial support by the DFG Cluster of Excellence “Machine Learning - New Perspectives for Science”, EXC 2064/1, project number 390727645; the German Federal Ministry of Education and Research (BMBF) through the Tübingen AI Center (FKZ: 01IS18039A); and funds from the Ministry of Science, Research and Arts of the State of Baden-Württemberg. The authors are also grateful to Magnus Waldemar Hoff Harder for alerting us to imprecisions in an early draft of this manuscript.
\end{ack}

\bibliography{references}
\bibliographystyle{icml2024}

\newpage
\appendix
\onecolumn
\input{appendix}

\input{checklist}

\end{document}

%% file: appendix.tex
\section{Recap}
\paragraph{Notation.}
Consider a function $f: \R^D \times \R^I \rightarrow \R^O$ with Jacobian $\J_{\w}(\x) = \partial_{\w} f_{\w}(\x)|_{\w = \w} \in \R^{O \times D}$ with respect to $\w$ evaluated in $\x$ and $\w$. For a given log-likelihood we define the Hessian w.r.t.\@ to the output $\Hess_{\w}(\x) = -\partial^2_{f_{\w}(\x)} \log p(\y | f_{w}(\x)) \in \R^{O \times O}$ and we assume it not to be dependent on $\y$ (which is true, for example, for exponential families).

Consider being given a dataset of finite size $N$, here we do not care about labels and we only refer to the collections of the datapoints $\mathcal{X}=\{\x_1,\ldots,\x_N\}\subset\R^I$.

Consider the stacking of the per-datum Jacobians $\J_{\w} = [\J_{\w}(\x_1); \ldots; \J_{\w}(\x_N)] \in \R^{NO \times D}$, where we dropped the dependence on $\mathcal{X}$. Similarly consider $\Hess_{\w} = \text{diag}(\Hess_{\w}(\x_1); \ldots; \Hess_{\w}(\x_N)) \in \R^{NO \times NO}$ the block diagonal stacking of the Hessians.

Consider the Generalized Gauss-Newton (\ggn) and the Neural Tangent Kernel (\ntk) matrices
\begin{equation}
    \ggn_\w = \J_\w^\top \Hess_{\w} \J_\w
    \in \R^{D \times D}
    \qquad\qquad
    \ntk_\w = \Hess_{\w}^{\nicefrac{1}{2}}\J_\w \J_\w^\top \Hess_{\w}^{\nicefrac{1}{2}}
    \in \R^{NO \times NO}.
\end{equation}

Recall also that the pullback pseudo-metric is $(f^*H)_\w = \ggn_\w$.

\paragraph{Assumptions.} 
We assume \emph{uniform} upper and lower bound on the eigenvalues of the \ntk\, matrix, that is
\begin{equation}
    \exists l,L\in\R
    \quad\textnormal{ such that }\qquad
    0 < l \leq \frac{\|\ \ntk_\w v\|}{\|v\|} \leq L
    \quad\forall v\in\R^{NO}, \forall\w\in\R^D,
\end{equation}
where uniform means uniform over parameters, i.e. the bounds $l,L$ holds for every $\w$. Moreover we assume that the Jacobian function $\w\mapsto J_\w(\x)$ is Lipschitz for every $\x\in\mathcal{X}$.

\paragraph{How unreasonable are the assumptions?} 
The assumption of an upper bound $L$ is equivalent to assuming that $\w\mapsto f(\w,\x)$ is Lipschitz for each datapoint $\x\in\mathcal{X}$. This is true with all standard activation functions if we restrict the parameter space to a ball of fixed radius.

The assumption of a lower bound $l$ is strongly supported by the literature on \ntk, thanks to its direct implications on memorization capacity and generalization. The general trend is that the more overparametrized the network is, the stronger such lower bounds are. In the hardest setting of minimum overparametrization, \citet{bombari2022memorization} proved a bound that holds with high probability for fully connected MLPs at initialization. Similar results hold with bigger overparametrizations \cite{nguyen2021tight, allen2019convergence, du2019gradient}. Building on top of that, other lines of work \cite{liu2020linearity, oymak2019overparameterized, oymak2020toward} proved that the \ntk\, does not change too much during training thanks to the PL inequality framework, and in particular the proximity of the neural network dynamics to the one described by \ntk, is supported by spectral bounds on the Hessian of the landscape.

Lastly, the Lipschitzness assumption on the Jacobian is potentially the most unrealistic, although it would hold, for example, if the derivatives of activations are Lipschitz and the parameters are restricted to a finite radius ball. Nonetheless, we emphasize that this assumption is only used to control that the kernel of the \ggn\, does not ``rotate too fast'', which is a much weaker assumption, but also much more cluttered to state formally.

\section{Proof for equivalence of the two settings}

This section contains the proof of 
\cref{prop:equivalence_quotient_and_pullback} and 
\cref{thm:equivalence_quotient_and_pullback} involving the pseudo-Riemannian manifold $(\mathbb{R}^D, f^*H)$ and the quotient group $(\mathcal{P}, d_\mathcal{P})$ with Euclidean-induced metric.

To ease the readability, we recall the two involved notions of distance and their respective definition:
\begin{itemize}
    \item $d_{f^*H}(\w_0,\w_1)$ the geodesic distance for any two parameters $\w_0,\w_1\in\R^D$
    \item $d_{\mathcal{P}}([\w_0,\w_1])$ the quotient Euclidean distance for any two equivalence classes $[\w_0],[\w_1]\in\mathcal{P}$
\end{itemize}
\begin{align}
    d_{f^*H}(\w_0,\w_1)
    &= \inf_\gamma \textsc{len}_{f^*H}(\gamma) \\
    &= \inf_\gamma \int_0^1 \|\gamma'(t)\|_{f^*H_{\gamma(t)}} \textnormal{d} t \\
    &= \inf_\gamma \int_0^1 \sqrt{\gamma'(t)^\top \cdot f^*H_{\gamma(t)} \cdot \gamma'(t)} \textnormal{d} t,
\end{align}
where the infimum is taken over smoothly differentiable curves $\gamma:[0,1]\rightarrow\R^D$ such that $\gamma(0)=\w_0$ and $\gamma(1)=\w_1$. 
\begin{align}
    d_\mathcal{P}([\w_0],[\w_1])
    =
    \inf
    \left\{
        \sum_{i=1}^n \|p_i-q_i\|
    \right\},
\end{align}
where the infimum is taken over all finite sequences $\{p_i\}_{i=1\ldots n},\{q_i\}_{i=1\ldots n}\subset\R^D$ such that $[\w_0]=[p_1]$, $[p_{i+1}]=[q_i]$ and $[q_n]=[\w_1]$.

Let us state a more quantitative theorem
\begin{theorem}
    For any $\w_0,\w_1\in\R^D$ it holds
    \begin{equation}\label{eq:reformulation_prop_equivalence}
        d_{f^*H}(\w_0,\w_1) = 0
        \quad\Longleftrightarrow\quad
        d_{\mathcal{P}}([\w_0,\w_1]) = 0,
    \end{equation}
    and also that, for any $\epsilon>0$ there exists $\delta>0$ such that
    \begin{align}
        d_{\mathcal{P}}([\w_0,\w_1]) < \delta
        & \quad\Longrightarrow\quad
        d_{f^*H}(\w_0,\w_1) < \epsilon \label{eq:reformulation_thm_equivalenceA}\\
        d_{f^*H}(\w_0,\w_1) < \delta
        & \quad\Longrightarrow\quad
        d_{\mathcal{P}}([\w_0,\w_1]) < \epsilon, \label{eq:reformulation_thm_equivalenceB}
    \end{align}
\end{theorem}
and note that by definition of bijection \cref{eq:reformulation_prop_equivalence} is equivalent to \cref{prop:equivalence_quotient_and_pullback} and, by definition of homeomorphism,  \cref{eq:reformulation_thm_equivalenceA} and \cref{eq:reformulation_thm_equivalenceB} are equivalent to \cref{thm:equivalence_quotient_and_pullback}.

We prove the 3 points separately, in \cref{sec:proof_prop_equivalence}, \cref{sec:proof_thm_equivalenceA} and \cref{sec:proof_thm_equivalenceB} respectively.

\subsection{Proof of \cref{eq:reformulation_prop_equivalence}}
\label{sec:proof_prop_equivalence}
The proof logic is 
\begin{equation}
    d_{f^*H}(\w_0,\w_1) = 0
    \quad\Longleftrightarrow\quad
    [\w_0] = [\w_1]
    \quad\Longleftrightarrow\quad
    d_{\mathcal{P}}([\w_0,\w_1]) = 0.
\end{equation}
and we prove the two steps in the two following Propositions, respectively.

\begin{proposition}\label{prop:same_equivalence_class_has_zero_geodesic_distance}
    For any $\w_0,\w_1\in\R^D$ it holds
    \begin{equation}
        d_{f^*H}(\w_0,\w_1) = 0
        \quad\Longleftrightarrow\quad
        [\w_0]=[\w_1].
    \end{equation}
\end{proposition}
\begin{proof}
    By definition $[\w_0]=[\w_1]$ if and only if there exists a piecewise differentiable $\gamma:[0,1]\rightarrow\R^D$ such that $\gamma(0)=\w_0$, $\gamma(1)=\w_1$ and $f(\w_0,\x)=f(\gamma(t),\x)$ for any $t\in[0,1]$ and $\x\in\mathcal{X}$. Then \\
    \boxed{\Longleftarrow} 
    Consider a $\gamma$ from the definition of the equivalence relation $\sim$ and define the points $\w_t=\gamma(t)$ for ease of notation. Then for any $t\in[0,1]$
    \begin{equation}
        \gamma'(t) = \lim_{\epsilon\rightarrow0} \frac{\w_{t+\epsilon}-\w_t}{\epsilon}.
    \end{equation}
    For any $\x\in\mathcal{X}$ it holds that $f(\w_t,\x)=f(\w_{t'},\x) \forall t,t'\in[0,1]$, which implies that $f(\w_{t+\epsilon},\x)-f(\w_t,\x)=0$ $\forall t\in[0,1]\forall\epsilon\in[0,1-t]$. Thus,
    \begin{equation}
        0
        =
        \lim_{\epsilon\rightarrow0} \frac{f(\w_{t+\epsilon},\x)-f(\w_t,\x)}{\epsilon}
        =
        \J_{\w_t}(\x) \cdot \lim_{\epsilon\rightarrow0} \frac{\w_{t+\epsilon}-\w_t}{\epsilon}
        =
        \J_{\w_t}(\x) \cdot \gamma'(t).
    \end{equation}
    This holds to any $\x\in\mathcal{X}$, so the same holds for the per-datum stacked jacobians $\J_{\w_t} \cdot \gamma'(t) = 0$. Thus,
    \begin{equation}
        \|\gamma'(t)\|^2_{f^*H_{\gamma(t)}}
        =
        \gamma'(t)^\top \cdot f^*H_{\gamma(t)} \cdot \gamma'(t)
         =
        \gamma'(t)^\top \cdot \J_{\w_t}^\top \Hess_{\w_t} \J_{\w_t} \cdot \gamma'(t)
        =
        0
    \end{equation}
    and we can measure the length of $\gamma$ in the pullback metric as
    \begin{equation}
        \textsc{len}_{f^*H}(\gamma)
        =
        \int_0^1 \|\gamma'(t)\|_{f^*H_{\gamma(t)}} \textnormal{d} t
        =
        \int_0^1 0 \,\textnormal{d} t
        =
        0,
    \end{equation}
    which gives an upper bound on the geodesic distance
    \begin{equation}
        d_{f^*H}(\w_0,\w_1)
        = \inf_{\hat{\gamma}} \textsc{len}_{f^*H}(\hat{\gamma})
        \leq
        \textsc{len}_{f^*H}(\gamma)
        =
        0,
    \end{equation}
    thus, $d_{f^*H}(\w_0,\w_1)=0$ and this implication is proven.

    \boxed{\Longrightarrow} $d_{f^*H}(\w_0,\w_1)=0$ implies that there exists a 0-length differentiable $\gamma:[0,1]\rightarrow\R^D$ such that $\gamma(0)=\w_0$, $\gamma(1)=\w_1$. Without loss of generality, we can assume $\gamma$ to be non-stationary, i.e. $\gamma'(t)\not=0$. Here 0-length means
    \begin{equation}
        0
        =
        \textsc{len}_{f^*H}(\gamma)
        =
        \int_0^1 \|\gamma'(t)\|_{f^*H_{\gamma(t)}} \textnormal{d} t,
    \end{equation}
    which implies that $\|\gamma'(t)\|_{f^*H_{\gamma(t)}} = 0$ for any $t\in[0,1]$ except a zero-measure set which we can neglect later. Then
    \begin{equation}
        0
        =
        \|\gamma'(t)\|^2_{f^*H_{\gamma(t)}}
        =
        \gamma'(t)^\top \cdot \J_{\w_t}^\top \Hess_{\w_t} \J_{\w_t} \cdot \gamma'(t)
        \quad
        \Longrightarrow
        \quad
        \J_{\w_t} \cdot \gamma'(t)=0,
    \end{equation}
    by positive definitess of $ \Hess_{\w_t}$, assumed as hypothesis. We highlight that the leftmost 0 in the previous equation is a scalar, while the rightmost 0 is a vector in $\R^{NO}$ as obtained by the matrix-vector product of $\J_{\w_t}\in\R^{NO\times D}$ with $\gamma'(t)\in\R^D$. Looking at the equation $\J_{\w_t} \cdot \gamma'(t)=0$ componentwise implies that
    \begin{equation}
        \langle \nabla_\w [f(\w_t, \x)]_o , \gamma'(t) \rangle 
        =
        0 \quad \forall \x\in\mathcal{X}, \forall o\in\{1,\ldots,O\}, \forall t\in[0,1],
    \end{equation}
    where $[v]_o$ refers to the $o$th component of a vector $v$. Thus, for $T\in[0,1]$, by Fundamental Theorem of Calculus we have
    \begin{equation}
        [f(\w_T, \x)]_o - [f(\w_0, \x)]_o
        =
        \int_0^T \langle \nabla_\w [f(\w_t, \x)]_o , \gamma'(t) \rangle  \textnormal{d}t
        = 0
        \qquad \forall \x\in\mathcal{X}, \forall o\in\{1,\ldots,O\}.
    \end{equation}
    Then $f(\w_T, \x) = f(\w_0, \x)\, \forall \x\in\mathcal{X}$ and $\forall T\in[0,1]$. So we proved that $\gamma$ is an homotopy of $(\w,\w')$ such that $\gamma(t)\in\mathcal{R}^f_{\mathcal{X}}(\w)\,\,\forall t\in[0,1]$. Thus,
    \begin{equation}
        \w'\in\bar{\mathcal{R}}^f_{\mathcal{X}}(\w)
        \quad\Longrightarrow\quad
        \w'\sim\w
        \quad\Longrightarrow\quad
        [\w']=[\w],
    \end{equation}
    and this completes the proof.
\end{proof}

\begin{proposition}
    For any $\w_0,\w_1\in\R^D$ it holds
    \begin{equation}
        [\w_0] = [\w_1]
        \quad\Longleftrightarrow\quad
        d_{\mathcal{P}}([\w_0,\w_1]) = 0.
    \end{equation}
\end{proposition}
\begin{proof}
    \boxed{\Longrightarrow} This arrow is trivially true by considering the two sequences in the definition of the quotient distance to be of length one and such that $p_1=\w_0$ and $q_1=\w_1$.
    
    \boxed{\Longleftarrow} $d_{\mathcal{P}}([\w_0,\w_1]) = 0$ implies that, by definition of $\inf$, there exists a sequence $\epsilon_m\rightarrow0$ such that $\forall m\in\mathbb{N}$ there exists two finite sequences of points $p_1^{(m)},\ldots,p_n^{(m)}\in\R^D$ and $q_1^{(m)},\ldots,q_n^{(m)}\in\R^D$ such that
    \begin{equation}
        \sum_{i=1}^n \|p_i^{(m)}-q_i^{(m)}\| = \epsilon_m,
    \end{equation}
    where $[\w_0]=[p_1^{(m)}]$, $[p_{i+1}^{(m)}]=[q_i^{(m)}]$ and $[q_n^{(m)}]=[\w_1]$. Note that the sequence length $n$ may depend on $m$.      \\
    Lipschitzness of $f$ in parameters, assumed by hypothesis, means that
    \begin{equation}
        \|p-q\| < \epsilon
        \quad\Longrightarrow\quad
        \|f(p,\x) - f(q,\x)\| < L\epsilon 
        \qquad
        \forall p,q\in\R^D,\forall\x\in\mathcal{X},\forall\epsilon>0,
    \end{equation}
    thus,
    \begin{equation}
        \sum_{i=1}^n \|p_i^{(m)}-q_i^{(m)}\| = \epsilon_m
        \quad\Longrightarrow\quad
        \sum_{i=1}^n \|f(p_i^{(m)},\x) - f(q_i^{(m)},\x)\| < 2L\epsilon_m
        \qquad
        \forall\x\in\mathcal{X}.
    \end{equation}
    Also, by definition of the equivalence class, it holds $\forall\x\in\mathcal{X}$ that
    \begin{align}
        [\w_0]=[p_1^{(m)}]
        &\quad\Longrightarrow\quad
        \|f(\w_0,\x) - f(p_1^{(m)},\x)\| = 0 \\
        [p_{i+1}^{(m)}]=[q_i^{(m)}]
        &\quad\Longrightarrow\quad
        \|f(p_{i+1}^{(m)},\x) - f(q_{i}^{(m)},\x)\| = 0 
        \quad
        \\
        [q_n^{(m)}]=[\w_1]
        &\quad\Longrightarrow\quad
        \|f(q_n^{(m)},\x) - f(\w_1,\x)\| = 0.
    \end{align}
    Thus, by the triangular inequality, the last four equations imply
    \begin{align}
      \begin{split}
        \|f(\w_0,\x) - f(\w_1,\x)\|
        & \leq
        \|f(\w_0,\x) - f(p_1^{(m)},\x)\|
        +
        \|f(p_1^{(m)},\x) - f(q_n^{(m)},\x)\|\\
        &+
        \|f(q_n^{(m)},\x) - f(\w_1,\x)\| 
      \end{split} \\
        & =
        \|f(p_1^{(m)},\x) - f(q_n^{(m)},\x)\| \\
        &\leq
        \sum_{i=1}^n \|f(p_i^{(m)},\x) - f(q_i^{(m)},\x)\|
        +
        \sum_{i=1}^{n-1} \|f(p_{i+1}^{(m)},\x) - f(q_i^{(m)},\x)\| \\
        & <
        2L\epsilon_m + \sum_{i=1}^{n-1} 0 
        = 2L\epsilon_m,
    \end{align}
    and this holds for any $m\in\mathbb{N}$. Taking the limit $m\rightarrow\infty$, $\epsilon_m\rightarrow0$ implies $\|f(\w_0,\x) - f(\w_1,\x)\|\leq0$. Thus, $f(\w_0,\x) = f(\w_1,\x)$ $\forall\x\in\mathcal{X}$, and, thus, $[\w_0] = [\w_1]$, which completes the proof.
\end{proof}

\subsection{Proof of \cref{eq:reformulation_thm_equivalenceA}}
\label{sec:proof_thm_equivalenceA}
In order to prove \cref{eq:reformulation_thm_equivalenceA}, we first prove a weaker statement

\begin{proposition}\label{prop:close_euclidean_imply_close_in_pullback_metric}
    For any $\w_0,\w_1\in\R^D$ and for any $\delta>0$ it holds that
    \begin{equation}
        \|\w_0 - \w_1\| < \delta
        \quad\Longrightarrow\quad
        d_{f^*H}(\w_0,\w_1) < L\delta .
    \end{equation}    
\end{proposition}
\begin{proof}
    Let $\gamma:[0,1]\rightarrow\R^D$ be defined as $\gamma(t)=(1-t)\w_0 + t\w_1$, then 
    \begin{equation}
      \begin{split}
        \textsc{len}_{f^*H}(\gamma)
       & =
        \int_0^1 \|\gamma'(t)\|_{f^*H_{\gamma(t)}} \textnormal{d} t 
        =
        \int_0^1 \|\w_1-\w_0\|_{f^*H_{\gamma(t)}} \textnormal{d} t  \\
       & \leq
        \int_0^1 L \|\w_1-\w_0\| \textnormal{d} t 
        =
        L \|\w_1-\w_0\|,
      \end{split}
    \end{equation}
    and thus,
    \begin{equation}
        d_{f^*H}(\w_0,\w_1)
        =
        \inf_{\hat{\gamma}} \textsc{len}_{f^*H}(\hat{\gamma})
        \leq
        \textsc{len}_{f^*H}(\gamma)
        =
        L \|\w_1-\w_0\|
        < L\delta.
    \end{equation}
\end{proof}

\begin{definition}
    Let $\gamma_1,\ldots,\gamma_n:[0,1]\rightarrow\R^D$ a sequence of piecewise differentiable paths such that
    \begin{equation}
        \gamma_{i-1}(1) = \gamma_{i}(0)
        \quad
        \forall i\in\{0,\ldots,n\}.
    \end{equation}
    Consider the piecewise differentiable path that is the \emph{concatenation} of the paths one after the other, $\hat{\gamma}=\textsc{cat}(\gamma_1,\ldots,\gamma_n)$ defined as
    \begin{equation}
        \hat{\gamma}(t) = \gamma_i\left(nt - i\right) 
        \quad\textnormal{ if } 
        i\leq t \leq i+1.
    \end{equation}
\end{definition}
It is straightforward to see that $\textsc{len}(\hat{\gamma}) = \sum_{i=1}^n\textsc{len}(\gamma_i)$.

The choice of $\epsilon,\delta$ that prove \cref{eq:reformulation_thm_equivalenceA} trivially follows from the following 
\begin{proposition}
    With the assumption of eigenvalue upper bound $L$, for any $\w_0,\w_1\in\R^D$ for any $\delta>0$ it holds that
    \begin{equation}
        d_{\mathcal{P}}([\w_0,\w_1]) < \delta
        \quad\Longrightarrow\quad
        d_{f^*H}(\w_0,\w_1) < 3L\delta .
    \end{equation}
\end{proposition}
\begin{proof}
    $d_{\mathcal{P}}([\w_0,\w_1]) < \delta$ implies that, by the definition of $\inf$, there exists two sequences $\{p_i\}_{i=1\ldots n},\{q_i\}_{i=1\ldots n}\subset\R^D$ such that $[\w_0]=[p_1]$, $[p_{i+1}]=[q_i]$, $[q_n]=[\w_1]$ and such that
    \begin{equation}
        \sum_{i=1}^n \|p_i-q_i\| < 2\delta.
    \end{equation}
    Now the idea is to define $n+1$ paths on the equivalence classes and $n$ paths connecting them, then the stacking of the $2n+1$ will give an upper bound on the geodesic distance.

    Let us first define the paths $\gamma_{2i}$ for $i=0,\ldots,n$ by making use of \cref{prop:same_equivalence_class_has_zero_geodesic_distance}
    \begin{itemize}
        \item $[\w_0]=[p_1]$ imply that there exists $\gamma_0:[0,1]\rightarrow\R^D$ be such that $\gamma_0(0)=\w_0$, $\gamma_0(1)=p_1$ and such that $\textsc{len}_{f^*H}(\gamma_0)<\nicefrac{L\delta}{n+1}$
        \item $[q_n]=[\w_1]$ imply that there exists $\gamma_{2n}:[0,1]\rightarrow\R^D$ be such that $\gamma_{2n}(0)=q_n$, $\gamma_{2n}(1)=\w_1$ and such that $\textsc{len}_{f^*H}(\gamma_{2n})<\nicefrac{L\delta}{n+1}$
        \item for all $i=1,\ldots,n-1$, $[p_{i+1}]=[q_i]$ imply that there exists $\gamma_{2i}:[0,1]\rightarrow\R^D$ be such that $\gamma_{2i}(0)=p_{i+1}$, $\gamma_{2i}(1)=q_i$ and such that $\textsc{len}_{f^*H}(\gamma_{2i})<\nicefrac{L\delta}{n+1}$
    \end{itemize}
    And then for all $i=1,\ldots,n$, \cref{prop:close_euclidean_imply_close_in_pullback_metric} and $\sum_{i=1}^n \|p_i-q_i\| < 2\delta$ imply $\gamma_{2i-1}:[0,1]\rightarrow\R^D$ be such that $\gamma_{2i-1}(0)=p_{i}$, $\gamma_{2i-1}(1)=q_i$ and such that $\sum_{i=1}^n \textsc{len}_{f^*H}(\gamma_{2i-1})<2L\delta$.

    Then concatenating these $2n+1$ paths, we have
    \begin{align}
        d_{f^*H}(\w_0,\w_1)
        & =
        \inf_{\hat{\gamma}} \textsc{len}_{f^*H}(\hat{\gamma})
        \leq
        \textsc{len}_{f^*H}(\textsc{cat}(\gamma_0,\ldots,\gamma_{2n})) \\
        & 
        =
        \sum_{i=0}^{2n} \textsc{len}_{f^*H}(\gamma_{i}) 
        = 
        \sum_{i=1}^n \textsc{len}_{f^*H}(\gamma_{2i-1})
        +
        \sum_{i=0}^n \textsc{len}_{f^*H}(\gamma_{2i})
        < 2L\delta + L\delta.
    \end{align}
\end{proof}

\subsection{Proof of \cref{eq:reformulation_thm_equivalenceB}}
\label{sec:proof_thm_equivalenceB}

\begin{proposition}\label{prop:control_on_metric_rotation}
    The pullback metric does not ``rotate'' too much when moving from $\w$ to $\w+\epsilon$, formally
    \begin{equation}
        (1 - K\|\epsilon\|)\|v\|_{f^*H_{\w+\epsilon}}
        \leq
        \|v\|_{f^*H_\w}
        \leq
        (1 + K\|\epsilon\|)\|v\|_{f^*H_{\w+\epsilon}}.
    \end{equation}
\end{proposition}
\begin{proof}
    Follows from the $K$-Lischitz assumption on $\w\mapsto\J_\w$ 
\end{proof}

The choice of $\epsilon,\delta$ that prove \cref{eq:reformulation_thm_equivalenceB} trivially follows from the following 
\begin{proposition}
    With the assumption of eigenvalue upper bound $L$, for any $\w_0,\w_1\in\R^D$ for any $\delta>0$ it holds that
    \begin{equation}
        d_{f^*H}(\w_0,\w_1) < \delta 
        \quad\Longrightarrow\quad
        d_{\mathcal{P}}([\w_0,\w_1]) < \frac{6\delta}{l}.
    \end{equation}
\end{proposition}
\begin{proof}
    $d_{f^*H}(\w_0,\w_1) < \delta$ implies that there exist a path $\gamma:[0,1]\rightarrow\R^D$ with 
    \begin{equation}
        2\delta
        =
        \textsc{len}_{f^*H}(\gamma)
        =
        \int_0^1 \|\gamma'(t)\|_{f^*H_{\gamma(t)}} \textnormal{d} t.
    \end{equation}
    Consider also the Euclidean length $\textsc{len}_E(\gamma)$ which does \emph{not} depend on $\delta$. Then, for any $\delta$ there exist $n(\delta)\in\mathbb{N}$ such that, considering the uniform partition of $[0,1]$, $t_i = \nicefrac{i}{n(\delta)}$ for $i=0,\ldots,n(\delta)$ it holds the discrete approximation of the integral
    \begin{align}
        3\delta 
        >
        \sum_{i=0}^{n(\delta)-1}
        \|\gamma(t_{i+1}) - \gamma(t_i)\|_{f^*H_{\gamma(t_i)}}
        & =
        \sum_{i=0}^{n(\delta)-1}
        \|\gamma(t_{i+1}) - p_i + p_i - \gamma(t_i)\|_{f^*H_{\gamma(t_i)}} \\
        &=
        \sum_{i=0}^{n(\delta)-1}
       \|\gamma(t_{i+1}) - p_i\|_{f^*H_{\gamma(t_i)}}
        + 
        \underbrace{\|p_i - \gamma(t_i)\|_{f^*H_{\gamma(t_i)}}}_{=0},
    \end{align}
    where $p_i$ is defined for every $i\in\{0,\ldots,n(\delta)-1\}$ as follow. Consider the projection $P^K_i$ on the kernel of $f^*H_{\gamma(t_i)}$, and the projection $P^I_i$ on the image of $f^*H_{\gamma(t_i)}$, such that the two projections are orthogonal and $P_i^K+P_i^I=\identity$. Define the point $p_i\in\R^D$ as $p_i=P_i^K(\gamma(t_{i+1}) - \gamma(t_i)) + \gamma(t_i)$, which implies that $\|p_i - \gamma(t_i)\|_{f^*H_{\gamma(t_i)}}=0$. By definition of the projections, it also hold that $p_i=\gamma(t_{i+1})-P_i^I(\gamma(t_{i+1}) - \gamma(t_i))$, which implies that $\gamma(t_{i+1}) - p_i = P_i^I(\gamma(t_{i+1}) - \gamma(t_i))$ is aligned with the non-zero eigenvalues of $f^*H_{\gamma(t_i)}$, this means that we can resort to the lower bound $l$ on the non-zero eigenvalues and \cref{prop:control_on_metric_rotation} to see that
    \begin{align}
        \|\gamma(t_{i+1}) - p_i\|_{f^*H_{\gamma(t_i)}} 
        &\geq
        (\underbrace{1 - K\|\gamma(t_{i+1})-\gamma(t_{i})\|}_{\tilde{K}})
        \|\gamma(t_{i+1}) - p_i\|_{f^*H_{\gamma(t_{i+1})}} \\
        &\geq
        \Tilde{K} l \|\gamma(t_{i+1}) - p_i\| \\
        &\geq
        \frac{l}{2} \|\gamma(t_{i+1}) - p_i\|
        \qquad\forall i\in\{0,\ldots,n(\delta)-1\},
    \end{align}
    where, recalling that $\|\gamma(t_{i+1})-\gamma(t_{i})\|=\nicefrac{1}{n(\delta)}$ by construction, there always exist an $n(\delta)$ big enough such that $\Tilde{K}\geq \nicefrac{1}{2}$, which we can assume without loss of generality. 
    Rearranging the terms, we have
    \begin{equation}
        \frac{6\delta}{l}
        \geq
        \sum_{i=0}^{n(\delta)-1}
        \|\gamma(t_{i+1}) - p_i\|.
    \end{equation}
    Finally, we can define the points $q_i=\gamma(t_{i+1})$ and we have the two sequences $p_0,\ldots,p_{n(\delta)-1}$ and $q_0,\ldots,q_{n(\delta)-1}$ whose satisfies the costrains $[\w_0]=[p_0]$, $[p_{i+1}]=[q_i]$ and $[q_{n(\delta)-1}]=[\w_1]$. We highlight that the dependence on $\delta$ of the length of the sequence is not problematic, as it is sufficient that $n(\delta)$ is finite for every fixed $\delta$, which it is. Thus,
    \begin{equation}
        d_{\mathcal{P}}([\w_0,\w_1])
        \leq \sum_{i=0}^{n(\delta)-1} \|p_i - q_i\|
        < 
        \frac{6\delta}{l},
    \end{equation}
    which concludes the proof.
\end{proof}

\section{Proof of existence of the two Riemannian manifolds}\label{sec: C}

This section contains the proof for the existence of the two Riemannian manifolds $(\mathcal{P}_{\bar{\w}},\m)$ and $(\mathcal{P}_{\bar{\w}}^\perp,\m^{\perp})$ embedded in $\mathbb{R}^D$. In the rest of the section, we consider a fixed $\bar{\w}\in\R^D$.

Fix a training set $\mathcal{X}=\{\x_1,\ldots,\x_N\}\subset\R^I$ of size $N$ and consider the stacked partial evaluation of the network defined as
\begin{align}
    \mathfrak{F}: 
    \mathbb{R}^D 
        & \longrightarrow \mathbb{R}^{NO} \\
    \w
        & \longmapsto 
        \left(f(\w, \x_1), \ldots, f(\w,\x_N)\right).
\end{align}
And define $\bar{\y}=\mathfrak{F}(\bar{\w})\in\R^{NO}$. The differential $\nabla_{\w}\mathfrak{F}\big|_{\bar{\w}} = \J_{\bar{\w}} \in\R^{NO\times D}$ equals the stacking of the per-datum Jacobians, and we assume it to be full rank thanks to the uniform lower bound on the eigenvalues of the \ntk\, matrix. Full-rankness in the overparametrized setting $D>NO$ implies the surjectivity of the differential operator. 
\begin{equation}
     \J_{\bar{\w}} 
     =
     \left(
        \frac{\partial \mathfrak{F}_i}{\partial \w_j}\bigg|_{\bar{\w}}
     \right)_{i=1,\ldots,NO;\, j=1,\ldots,D}.
\end{equation}
We reorder the $\w_i$ so that the first $NO$ columns are independent. Then the $NO\times NO$ matrix
\begin{equation}
     R
     =
     \left(
        \frac{\partial \mathfrak{F}_i}{\partial \w_j}\bigg|_{\bar{\w}}
     \right)_{i=1,\ldots,NO;\, j=1,\ldots,NO}
\end{equation}
is non-singular. We consider the map
\begin{equation}
    \alpha(\w_1, \ldots, \w_D)
    =
    \left( 
        \mathfrak{F}(\w)_1,
        \ldots,
        \mathfrak{F}(\w)_{NO},
        \w_{NO+1},
        \ldots,
        \w_D
    \right).
\end{equation}
We obtain
\begin{equation}
    \nabla_\w \alpha\big|_{\bar{\w}}
    =
    \left(
        \frac{\partial \alpha_i}{\partial \w_j}\bigg|_{\bar{\w}}
    \right)_{i=1,\ldots,D;\, j=1,\ldots,D}
    =
    \begin{pmatrix}
        R & * \\
        0 & \mathbb{I}
    \end{pmatrix},
\end{equation}
and this is non-singular. By the inverse function theorem, $\alpha$ is a local diffeomorphism. So there is an open $W\subseteq\R^D$ containing $\bar{w}$ such that $\alpha|_W: W\rightarrow\alpha(W)$ is smooth with smooth inverse.

Finally, define 
\begin{equation}
    \mathcal{P}_{\bar{\w}}^\perp = \left\{ \alpha^{-1}(\underbrace{\bar{\y}_1,\ldots,\bar{\y}_{NO}}_{NO}, p_{1}, \ldots, p_{D-NO}) \quad \textnormal{ for } p\in\R^{D-NO} \right\}\subseteq\R^D,
\end{equation}
and similarly 
\begin{equation}
    \mathcal{P}_{\bar{\w}} = \left\{ \alpha^{-1}(p_{1}, \ldots, p_{NO}, \underbrace{0,\ldots,0}_{D-NO}) \quad \textnormal{ for } p\in\R^{NO} \right\}\subseteq\R^D.
\end{equation}

We claim that the two restrictions of $\alpha$ are slice charts of $\mathcal{P}_{\bar{\w}}^\perp$ and $\mathcal{P}_{\bar{\w}}$, respectively. Since it is a smooth diffeomorphism, it is certainly a chart. Moreover, by construction, the points in $\mathcal{P}_{\bar{\w}}^\perp$ are exaclty those whose image under $\mathfrak{F}$ is $\bar{y}$, thus, $\mathcal{P}_{\bar{\w}}^\perp\subseteq[\bar{\w}]$.
On the other hand, the points $\w\in\mathcal{P}_{\bar{\w}}$ parametrize functions that take the values $\mathfrak{F}(\w)=p$ in a local neighbourhood of $\bar{y}$. Thus, locally it never intersects the same equivalence class more than one time.

\section{Proof of Theorem \ref{thm:5.1}}
By definition of the kernel manifold, we have that if
    $\w \in \mathcal{P}_{\w'}^\perp$ then we have that $\w \in [\w']$. Hence for all $\w \in \mathcal{P}_{\w'}^\perp$ and for all $\x \in \mathcal{X}$
\begin{equation}
    f(\w, \x) = f(\w', \x).
\end{equation}
It follows that $\mathrm{Var}_{\w\sim \mathcal{P}_{\w'}^\perp}\left[ f(\w,\x)\right] = 0$ for any $\x\in\mathcal{X}$. 

For the second statement notice that $\mathrm{Var}_{\w\sim\mathcal{P}_{\w'}^\perp}\left[ f(\w,\x_{test} \right] = 0 $ if and only if $f(\w,\x_{test}) = \bm{c}$ for all $\w \sim \mathcal{P}_{\w'}^\perp$ and for some constant $\bm{c}$. 

Suppose $\hat{\w}\in\bar{\mathcal{R}}^f_{\mathcal{X}}$ and $\hat{\w}\not\in\bar{\mathcal{R}}^f_{\mathcal{X}\cup\{\x_{test}\}}$, then $f(\hat{\w}, \x_{test}) \neq f(\w', \x_{test})$, which means that $ f(\w,\x_{test})$ is not constant. Hence we have that $\mathrm{Var}_{\w\sim\mathcal{P}_{\w'}^\perp}\left[ f(\w,\x_{test} \right] > 0 $

\section{Further results and experimental setup}

\subsection{Implementation details of the Laplace approximation}\label{sec: numerics}
Sampling from Laplace's approximation requires computing the inverse square root of a matrix of size $D \times D$, where $D$ is the number of parameters. For most models this problem is intractable. The standard approach to this problem is to consider sparse approximations to the Hessian of the loss function such as KFAC, Last-Layer, and Diagonal approximations. However, these approximations introduce additional complexity making the task of validating our theoretical analysis much harder. In light of these considerations, we choose to sample from Laplace's approximation in a way that is closest to the theoretical ideal, at the cost of performing expensive computations.

\begin{wrapfigure}[14]{r}{0.4\textwidth}  \includegraphics[width=\linewidth]{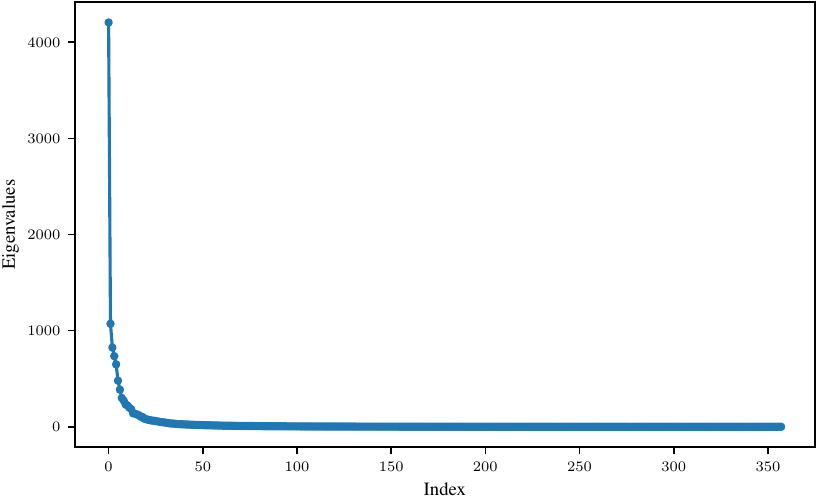}
  \caption{Eigenvalues of the $\ggn$ of a Convolutional Neural Network trained on MNIST.}
  \label{fig:eigvals}
\end{wrapfigure}

In small experiments with the toy regression problem, we instantiate the exact $\ggn$ and compute vector products with its inverse-square root. For experiments with LeNet and ResNet, we rely on the empirical observation that the spectrum of the $\ggn$ is dominated by its leading eigenvalues (Figure~\ref{fig:eigvals})

This makes low-rank approximations of the $\ggn$ particularly attractive. We choose the Lanczos algorithm \citep{lanczos:hal-01712947} with full reorthogonalization and run it for a very high number of iterations, to ensure numerical stability and very low reconstruction error to form our low-rank approximations of the $\ggn$. Additionally, Lanczos only requires implicit access to the matrix so we avoid the memory cost and $\ggn$ vector products for neural networks can be performed efficiently using Jacobian-Vector products and Vector-Jacobian products. If we do a sufficiently large number of iterations we obtain the non-zero eigenvalues $\tilde{\mat{\Lambda}}$, and corresponding eigenvectors $\mat{U_1}$.
For obtaining samples from the diffusion we can use Algorithm~\ref{alg:training_alg} with the eigenvalues and eigenvectors computed using the Lanczos algorithm.
 
 \begin{algorithm}[t]
    \centering
    \caption{Laplace diffusion}
    \label{alg:training_alg}
    \begin{algorithmic}[1]
        \STATE {\textbf{Input:}} Observed data $\mathcal{D}$, trained \textsc{map} point $\w'$, number of steps $\bm{T}$, number of samples $\bm{S}$, rank $k$.
        \STATE Initialize samples $\w^0_1 \ldots \w^0_{\bm{S}}$ as the \textsc{map} estimate $\w'$
        \FOR{$j$ \textbf{in} $1, \ldots \bm{S}$}
        \FOR{$t$ \textbf{in} $1, \ldots \bm{T}$}
        \STATE {Sample $\epsilon \sim \mathcal{N}(0,\mathbf{I}_{k})$}
        \STATE {Compute the top-$k$ eigenvalues($\Lambda^t_j$) and eigenvectors($U^t_j$) of $\ggn_{\w^t_{j}}$}
        \STATE {$\w^t_{j} \leftarrow \w^t_{j} + \frac{1}{\sqrt{\bm{N}}} U^t_j (\Lambda^t_j + \alpha\mat{I})^{-\frac{1}{2}} \epsilon$}
        \ENDFOR
        \ENDFOR
        \STATE Return posterior samples $\w^{\bm{T}}_1 \ldots \w^{\bm{T}}_{\bm{S}}$.
    \end{algorithmic}
\end{algorithm}

 Given the non-zero eigenvalues $\tilde{\mat{\Lambda}}$, and corresponding eigenvectors $\mat{U_1}$ we can also form inverse vector products with the square root of $\ggn + \alpha I$. It should be evident from the discussion in section~\ref{reparam} about decomposing the covariance that this vector product with a vector $v$ is given by:
\begin{align}
    (\ggn + \alpha I)^{-\frac{1}{2}} v &= \mat{U_1} (\tilde{\mat{\Lambda}} + \alpha \mat{I}_k)^{-\frac{1}{2}} v + \frac{1}{\sqrt{\alpha}}\mat{U_2} v \\
            &= \mat{U_1} (\tilde{\mat{\Lambda}} + \alpha \mat{I}_k)^{-\frac{1}{2}} v + \frac{1}{\sqrt{\alpha}}(\mat{I} - \mat{U_1}) v \\
            &= \mat{U_1} ((\tilde{\mat{\Lambda}} + \alpha \mat{I}_k)^{-\frac{1}{2}} - \frac{1}{\sqrt{\alpha}} \mat{I}_k)v + \frac{1}{\sqrt{\alpha}} v
\end{align}

This allows us to form inverse-square root vector products with the $\ggn  + \alpha I$ given the non-zero eigenspectrum. We run the sampling algorithm on H100 GPUs to run the high-order Lanczos decomposition. This approach of sampling from Laplace's approximation has $O(pk^2) $ time complexity, and $O(pk)$ memory cost, where $k$ is the number of Lanczos iterations and  $p$ is the number of parameters.

\subsection{Experimental details and further results for toy experiments}\label{sec: toy_results}

\subsubsection{Toy regression in Figure~\ref{fig:laplace_fails}}
In this experiment, we fit a small MLP, with 2 hidden layers of width 10  on the sine curve. Due to the small size of the $\ggn$ it is possible to instantiate and do all the computations explicitly. We sample from the exact Laplace's approximation, the non-kernel and kernel subspace of the $\ggn$, and use the neural network and the linearized predictive functions for the top row and the bottom row respectively. For the middle, we simulate a diffusion on the kernel manifold for the center plot, a diffusion in the non-kernel manifold for the right plot and we do alternating steps in the two manifolds for the left plot which gives us the full distribution.

\begin{figure}[t]
  \includegraphics[width=\linewidth]{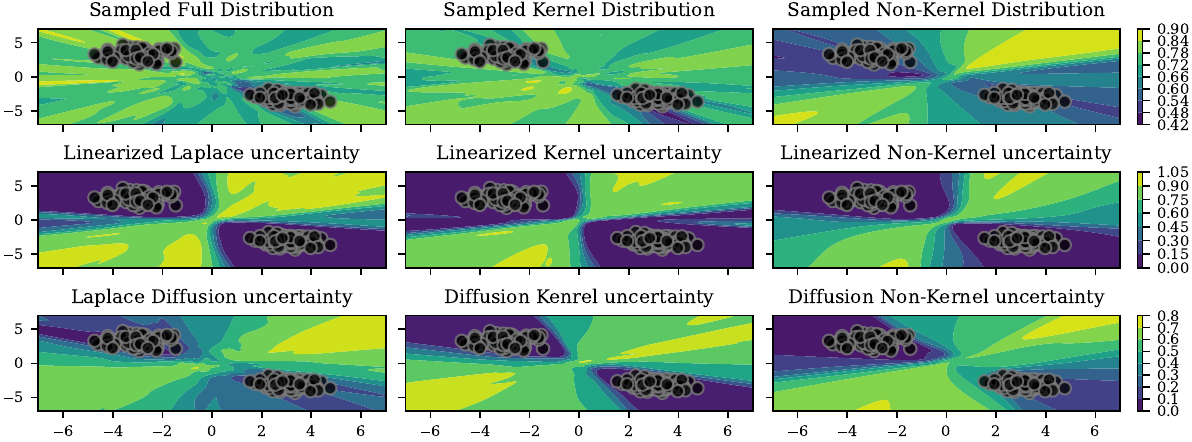}
  \vspace{-7mm}
  \caption{Decomposition of uncertainties of Laplace Approximation for the Gaussian mixture classification.}
  \label{fig:laplace_fails_classification}
\end{figure}

We also do a similar experiment to show that the same phenomenon also holds for classification. We use a small convolutional neural network, with 2 convolutional layers with kernel of size 3, to classify a 2-class mixture of Gaussians and look at uncertainties of sampled Laplace, linearized Laplace, and Laplace's diffusion. We decompose these uncertainties into their kernel and non-kernel components respectively. We see the same effect for classification as we did in regression in Figure~\ref{fig:laplace_fails_classification}

The main takeaway of these experiments is that sampled Laplace underfits in-distribution and this effect is related to the kernel component of the distribution. 

\subsubsection{Effect of kernel rank on in-distribution fit in figure~\ref{fig:rank_plot}}

For this experiment, we train a small convolutional neural network, with two convolutional layers and a kernel of size 3, on MNIST. In this case, the $\ggn$ can be instantiated explicitly. We recall that the $\ggn$ is a sum of $\J_{\w'}(\x_n)\T \Hess(\x_n) \J_{\w'}(\x_n)$ over the dataset where $\x_n$ are the individual data points. We recall that the rank of the $\ggn$ is bounded by $NO$ where $N$ is the number of data points and $O$ is the output dimensions. This suggests considering partial sums by subsampling the data points gives us a $\ggn$ with a lower rank. Equivalently, this is $\ggn$ with a higher dimensional kernel, and hence the usual covariance from Laplace's approximation $(\ggn + \alpha I)^{-1}$ has a higher contribution from the kernel subspace.

We consider multiple such subsamples and plot the training accuracy for samples from Laplace's approximation against the kernel subspace dimension. Here we see a clear trend that the underfitting in sampled Laplace decreases as the rank of $\ggn$ increases, or the contribution from the kernel component decreases. This serves to further support our suggested hypothesis that the underfitting in sampled Laplace is caused by its kernel component and is hence deeply related to reparameterizations.

\subsection{Additional benchmarks and results for image classification}\label{sec: more baselines}

\paragraph{Training details:} 
We use a standard LeNet for the MNIST and FashionMNIST experiments and a smaller version of ResNet \citep{lippe2022uvadlc}, with 272,378 parameters consisting of 3 times a group of 3 ResNet blocks. We use those instead of the standard ResNets due to constraints on the computational budget for the CIFAR-10 experiments.  We train LeNet with Adam optimizer and a learning rate of $10^{-3}$. For the ReNet we use SGD with a learning rate of $0.1$ with momentum and weight decay.

\paragraph{Hyperparameters:} We benchmark Laplace diffusion against SWAG, diagonal Laplace, last-layer Laplace, and MAP in addition to linearised Laplace and sampled Laplace. 

For choosing the prior precision for diagonal Laplace and last-Layer Laplace for each benchmark we do a grid search over the set $\{0.1, 1.0, 5.0, 10.0, 50.0, 100.0 \}$. For Laplace diffusion, sampled Laplace, and Linearised to ensure that the comparison can validate the theory it is preferable to have the same prior precision for all of these methods. So we only do the grid search to tune the prior precision for sampled Laplace and use this for all three methods. We keep the hyperparameters for these three methods as similar as possible to have the most informative comparisons. 

For Laplace diffusion on MNIST and FMNIST, we simulate the diffusion with a step size of $0.05$, with $2000$ Lanczos iterations and we predict using $20$ MC samples. For sampled Laplace and Linearised Laplace, we also use $2000$ Lanczos iterations and we predict using $20$ MC samples. For the CIFAR-10 experiments, we simulate the diffusion with a step size of $0.2$, with $5000$ Lanczos iterations and we predict using $20$ MC samples. For sampled Laplace andlLinearised Laplace use the same number of Lanczos iterations and MC samples. 

For SWAG we use a learning rate of $10^{-2}$ with momentum of $0.9$ and weight decay of $3e^{-4}$ and the low-rank covariance structure in all experiments. For the MNIST and FMNIST experiments we collect 20 models and for the CIFAR-10 experiments we collect 3 models to sample from the posterior.

Last-layer Laplace is the recommended method by \citet{daxberger2021bayesian} so it should approximate the best performance one can get using various possible configurations. For the CIFAR-10 experiments, the last layer of ResNet is too large to instantiate the full $\ggn$ matrix. So we instead use the last 1000 parameters of the model to construct the covariance matrix of the posterior.

Diagonal Laplace requires high prior precision to ensure it does not severely underfit in-distribution (similar to sampled Laplace). It often becomes almost deterministic. So we exclude it from the CIFAR results. This has also been observed by \citet{deng2022accelerated} and \citet{ritter2018scalable}. 

All additional information about the experimental setup can be found in the submitted code.

\subsubsection{In-distribution fit and calibration}
We extend Table~\ref{tab:id} to benchmark Laplace's diffusion against various other Bayesian methods. Here we see that despite using the neural network predictive it is competitive with the best-performing Bayesian methods whereas Sampled Laplace performs significantly worse.

\paragraph{MNIST}

  \setlength{\aboverulesep}{0pt}
  \setlength{\belowrulesep}{0pt}
  \setlength{\extrarowheight}{.75ex}
\resizebox{\textwidth}{!}{\begin{tabular}{lllllll}
\toprule
\rowcolor{gray!30}
                            & Conf.~($\uparrow$)           & NLL~($\downarrow$)          & Acc.~($\uparrow$)           & Brier~($\downarrow$)        & ECE~($\downarrow$)        & MCE~($\downarrow$) \\
\midrule
Laplace's diffusion & 0.988±0.001 & 0.040±0.007 & 0.986±0.002 & 0.022±0.003 & 0.146±0.011 & 0.773±0.050 \\
Sampled Laplace & 0.593±0.002 & 3.669±0.157 & 0.157±0.031 & 1.162±0.044 & 0.436±0.027 & 0.984±0.001 \\
Linearised Laplace & 0.967±0.002 & 0.295±0.041 & 0.930±0.004 & 0.111±0.006 & 0.239±0.032 & 0.862±0.044 \\
SWAG & 0.993±0.001 & 0.032±0.001 & 0.989±0.002 & 0.019±0.001 & 0.187±0.024 & 0.901±0.041 \\
Last-Layer Laplace & 0.991±0.001 & 0.047±0.003 & 0.987±0.002 & 0.021±0.001 & 0.198±0.034 & 0.721±0.091 \\
Diagonal Laplace & 0.963±0.005 & 0.078±0.011 & 0.976±0.003 & 0.038±0.005 & 0.095±0.004 & 0.692±0.029 \\
MAP & 0.992±0.002 & 0.042±0.004 & 0.988±0.000 & 0.021±0.000 & 0.198±0.041 & 0.685±0.091 \\
\bottomrule
\end{tabular}}

\paragraph{FMNIST}

\resizebox{\textwidth}{!}{\begin{tabular}{lllllll}
\toprule
 \rowcolor{gray!30}
                            & Conf.~($\uparrow$)           & NLL~($\downarrow$)          & Acc.~($\uparrow$)           & Brier~($\downarrow$)        & ECE~($\downarrow$)        & MCE~($\downarrow$) \\
\midrule
Laplace's diffusion & 0.900±0.001 & 0.275±0.016 & 0.906±0.007 & 0.141±0.006 & 0.108±0.015 & 0.729±0.092 \\
Sampled Laplace & 0.618±0.021 & 4.507±0.160 & 0.098±0.010 & 1.295±0.014 & 0.518±0.013 & 0.986±0.001 \\
Linearised Laplace & 0.897±0.003 & 0.423±0.014 & 0.862±0.005 & 0.207±0.006 & 0.147±0.017 & 0.756±0.048 \\
SWAG & 0.925±0.002& 0.259±0.004 & 0.911±0.006 & 0.135±0.004 & 0.152±0.008 & 0.752±0.067 \\
Last-Layer Laplace & 0.914±0.001 & 0.280±0.016 & 0.901±0.006 & 0.144±0.004 & 0.131±0.003 & 0.673±0.083 \\
Diagonal Laplace & 0.862±0.009 & 0.323±0.022 & 0.889±0.010 & 0.165±0.011 & 0.102±0.003 & 0.660±0.051 \\
MAP & 0.914±0.000 & 0.279±0.015 & 0.904±0.006 & 0.143±0.004 & 0.125±0.010 & 0.609±0.033 \\
\bottomrule
\end{tabular}
}

\paragraph{CIFAR-10}

\resizebox{\textwidth}{!}{\begin{tabular}{lllllll}
\toprule
 \rowcolor{gray!30}
                            & Conf.~($\uparrow$)           & NLL~($\downarrow$)          & Acc.~($\uparrow$)           & Brier~($\downarrow$)        & ECE~($\downarrow$)        & MCE~($\downarrow$) \\
\midrule
Laplace's diffusion & 0.948±0.004 & 0.403±0.007 & 0.889±0.005 & 0.180±0.009 & 0.264±0.048 & 0.879±0.038 \\
Sampled Laplace & 0.843±0.004 & 0.997±0.222 & 0.717±0.049 & 0.422±0.081 & 0.221±0.047 & 0.804±0.080 \\
Linearised Laplace & 0.951±0.007 & 0.614±0.020 & 0.863±0.001 & 0.222±0.002 & 0.337±0.022 & 0.789±0.035 \\
SWAG & 0.942±0.003 & 0.393±0.004 & 0.884±0.002 & 0.176±0.001 & 0.234±0.014 & 0.912±0.016 \\
Last-Layer Laplace & 0.953±0.004 & 0.343±0.033 & 0.899±0.001 & 0.155±0.007 & 0.260±0.000 & 0.884±0.029 \\
MAP & 0.960±0.003 & 0.333±0.030 & 0.913±0.007 & 0.143±0.009 & 0.282±0.002 & 0.932±0.006 \\
\bottomrule
\end{tabular}
}

\subsubsection{Robustness to dataset shift}\label{sec:robustness}

\begin{figure}[h!]

\centering
\includegraphics[width=\textwidth]{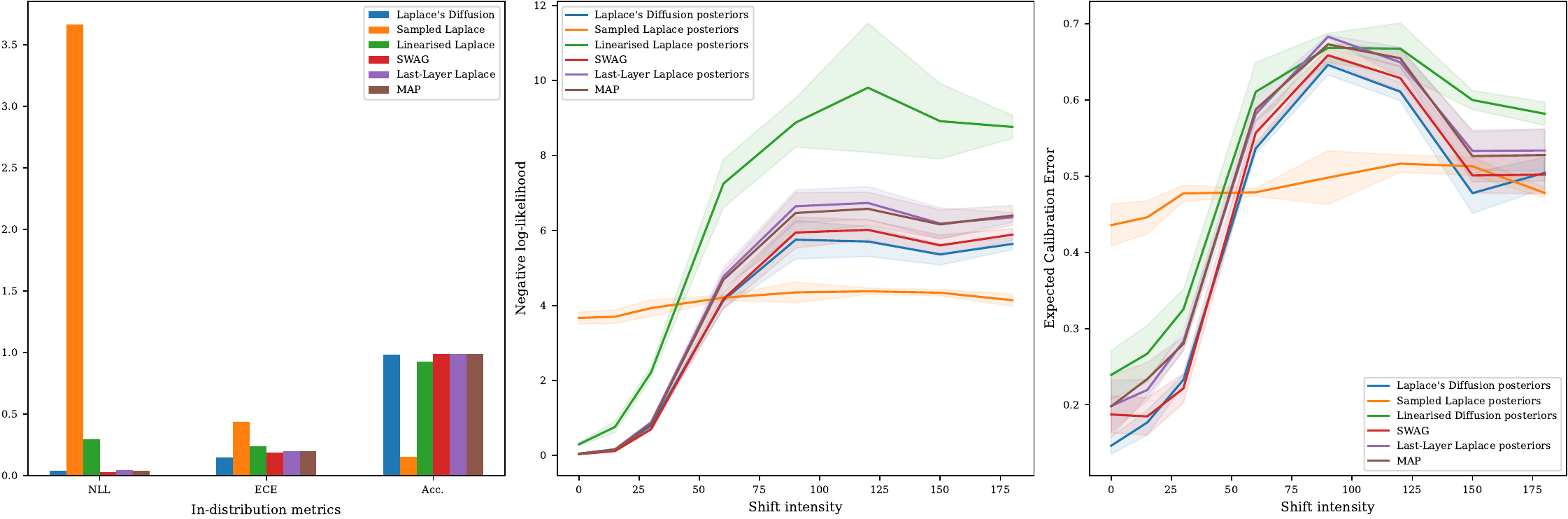}\hfill
\includegraphics[width=\textwidth]{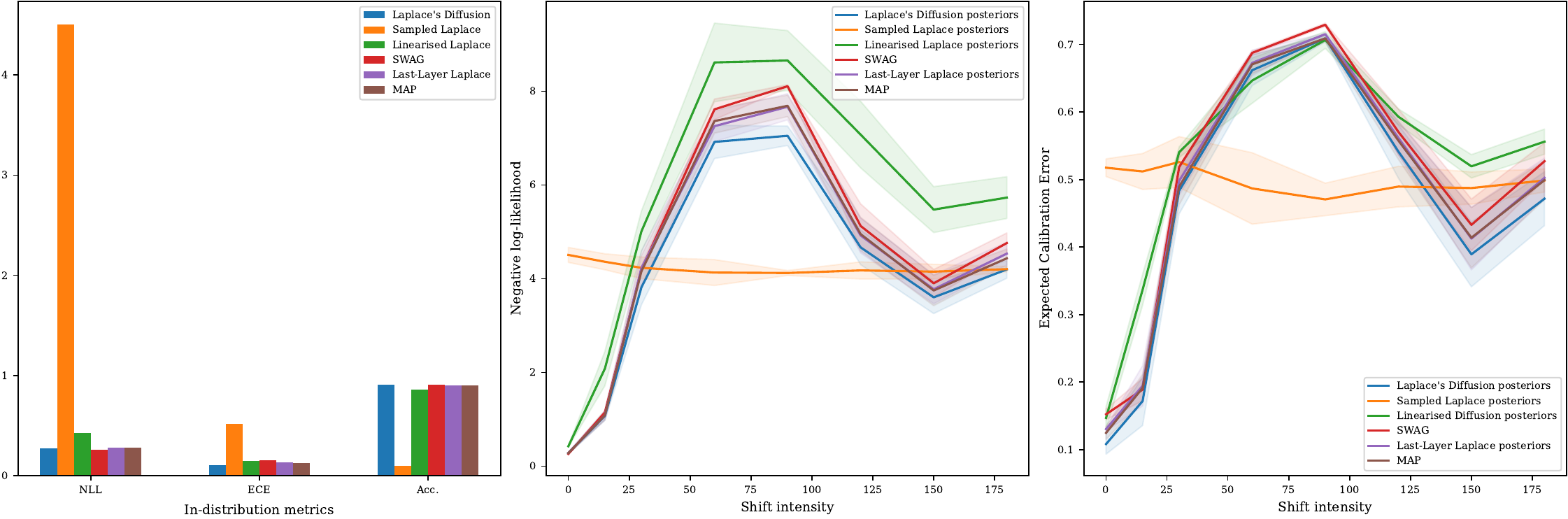}\hfill
\includegraphics[width=\textwidth]{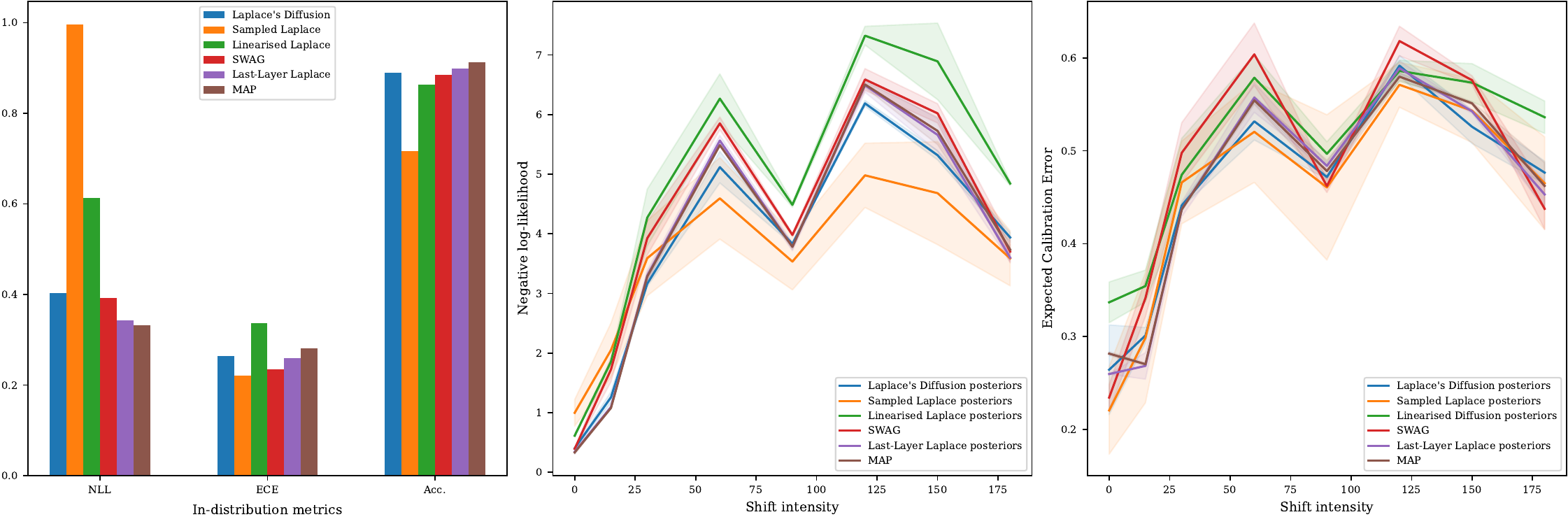}

\caption{Model Fit and Calibration of various posterior sampling methods on in-distribution data(first column) and under distribution shift for MNIST(top row), Fashion MNIST(middle row) and CIFAR-10(bottom row). We use rotated MNIST, rotated FMNIST, and rotated CIFAR in the second and third columns. SHift intensities denote angles of rotation.}
\label{fig: baseline  rotations}

\end{figure}

In these experiments (Fig.~\ref{fig: baseline  rotations}), to measure in-distribution fit and calibration, we report accuracy, negative log-likelihood (NLL), and expected calibration error (ECE)—all evaluated on the standard test sets. We measure the robustness of dataset shift of various baselines by plotting the negative log-likelihood and the expected calibration error against shift intensity. The desired behavior is good in-distribution fit, as close as possible to MAP, and stable calibration errors and NLL under distribution shifts. We see that the Laplace's diffusion is competitive against other Bayesian methods.






\subsubsection{Out-of-distribution detection}
We extend ~\ref{tab:ood} to benchmark Laplace's diffusion against various other Bayesian methods for Out-of-Distribution Detection. Once again, we observe that, despite using the neural network predictive, it is competitive with the best-performing Bayesian methods in terms of having a higher AUROC, whereas Sampled Laplace performs significantly worse. 

\paragraph{MNIST}

\resizebox{\textwidth}{!}{
\begin{tabular}{lllllll}
\toprule
\rowcolor{gray!30}
 Tested on & \multicolumn{2}{r}{FMNIST} & \multicolumn{2}{r}{EMNIST} & \multicolumn{2}{r}{KMNIST} \\
 \rowcolor{gray!30}
 & Conf.($\downarrow$) & AUROC($\uparrow$) & Conf.($\downarrow$) & AUROC($\uparrow$) & Conf.($\downarrow$) & AUROC($\uparrow$) \\
\midrule
Laplace's diffusion & 0.810±0.031 & 0.911±0.037 & 0.947±0.007 & 0.629±0.021 & 0.817±0.015 & 0.930±0.009 \\
Sampled Laplace & 0.583±0.015 & 0.515±0.021 & 0.584±0.003 & 0.515±0.004 & 0.588±0.013 & 0.507±0.015 \\
Linearised Laplace & 0.895±0.027 & 0.715±0.086 & 0.934±0.004 & 0.536±0.028 & 0.863±0.009 & 0.757±0.018 \\
SWAG & 0.827±0.054 & 0.949±0.018 & 0.955±0.004 & 0.627±0.016 & 0.823±0.013 & 0.947±0.007 \\
Last-Layer Laplace & 0.842±0.046 & 0.904±0.033 & 0.958±0.005 & 0.625±0.028 & 0.837±0.009 & 0.935±0.004 \\
Diagonal Laplace & 0.770±0.026 & 0.866±0.030 & 0.907±0.005 & 0.617±0.009 & 0.747±0.013 & 0.897±0.005 \\
MAP & 0.850±0.035 & 0.907±0.038 & 0.959±0.005 & 0.634±0.019 & 0.837±0.010 & 0.938±0.006 \\
\bottomrule
\end{tabular}}

\newpage

\paragraph{FMNIST}

\resizebox{\textwidth}{!}{\begin{tabular}{lllllll}
\toprule
\rowcolor{gray!30}
 Tested on & \multicolumn{2}{r}{MNIST} & \multicolumn{2}{r}{EMNIST} & \multicolumn{2}{r}{KMNIST} \\
 \rowcolor{gray!30}
 & Conf.($\downarrow$) & AUROC($\uparrow$) & Conf.($\downarrow$) & AUROC($\uparrow$) & Conf.($\downarrow$) & AUROC($\uparrow$)  \\
\midrule
Laplace's diffusion & 0.734±0.039 & 0.759±0.045 & 0.730±0.020 & 0.741±0.010 & 0.746±0.012 & 0.749±0.023 \\
Sampled Laplace & 0.597±0.027 & 0.495±0.037 & 0.593±0.026 & 0.503±0.036 & 0.598±0.024 & 0.493±0.033 \\
Linearised Laplace & 0.817±0.022 & 0.625±0.050 & 0.813±0.003 & 0.628±0.013 & 0.816±0.014 & 0.624±0.020 \\
SWAG & 0.727±0.013 & 0.817±0.015 & 0.763±0.028 & 0.769±0.026 & 0.777±0.007 & 0.782±0.010 \\
Last-Layer Laplace & 0.757±0.019 & 0.761±0.031 & 0.760±0.017 & 0.735±0.018 & 0.772±0.008 & 0.747±0.021 \\
Diagonal Laplace & 0.652±0.033 & 0.767±0.032 & 0.682±0.019 & 0.719±0.033 & 0.696±0.025 & 0.728±0.032 \\
MAP & 0.759±0.019 & 0.757±0.032 & 0.762±0.018 & 0.730±0.015 & 0.773±0.010 & 0.743±0.024 \\
\bottomrule
\end{tabular}}

\paragraph{CIFAR-10}

\begin{tabular}{lllll}
\toprule
\rowcolor{gray!30}
 Tested on & \multicolumn{2}{c}{CIFAR-100} & \multicolumn{2}{c}{SVHN} \\
 \rowcolor{gray!30}
 & Conf.($\downarrow$) & AUROC($\uparrow$) & Conf.($\downarrow$) & AUROC($\uparrow$) \\
\midrule
Laplace's diffusion & 0.790±0.002 & 0.851±0.002 & 0.764±0.008 & 0.862±0.010 \\
Sampled Laplace & 0.727±0.026 & 0.687±0.033 & 0.792±0.022 & 0.599±0.038 \\
Linearised Laplace & 0.818±0.005 & 0.837±0.006 & 0.809±0.033 & 0.854±0.024 \\
SWAG & 0.776±0.009 & 0.845±0.008 & 0.729±0.009 & 0.876±0.003 \\
Last-Layer Laplace & 0.789±0.006 & 0.864±0.001 & 0.786±0.029 & 0.868±0.015 \\
MAP & 0.797±0.004 & 0.873±0.001 & 0.792±0.034 & 0.878±0.017 \\
\bottomrule
\end{tabular}

\subsection{Short discussion on benchmarks}

\paragraph{Sparse Approximations of $\ggn$.}
Our theoretical analysis is mainly concerned with the ideal versions of Laplace's approximations, where we consider the full $\ggn$ in the covariance without any approximations. However, it can also shed some light on other Bayesian methods. 

It is common to use sparse approximations of the $\ggn$ when doing Laplace's approximations. Interestingly we observe that Laplace's approximations with sparse $\ggn$ such as diagonal Laplace, Last Layer, etc do not benefit from linearization to the same degree (Fig.~\ref{fig:diag_vs_full}).

\begin{figure}[H]
  \centering
  \includegraphics[width=0.7\linewidth]{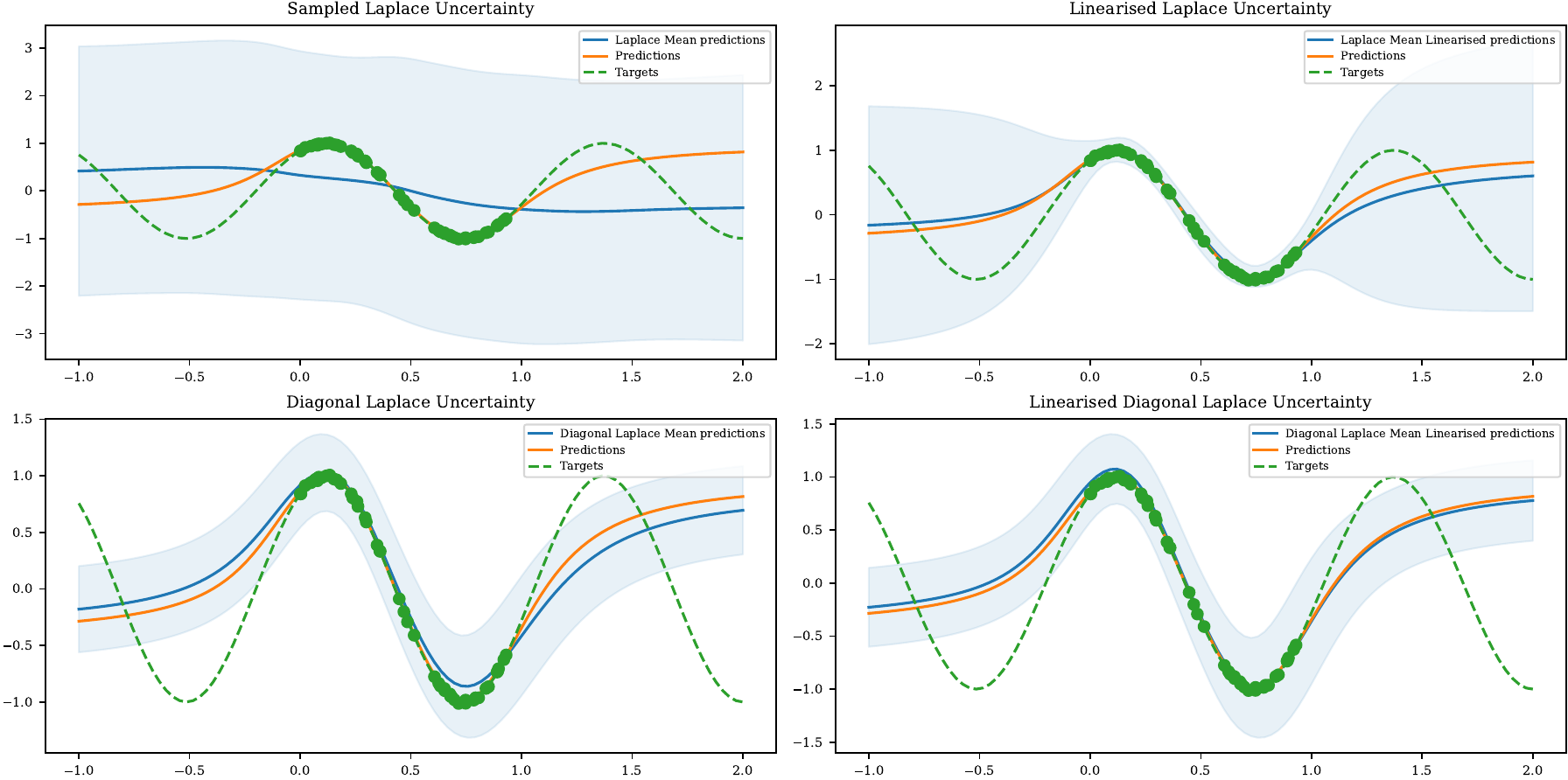}
  \caption{Predictive uncertainty of Laplace's approximation with neural network and linearized predictive (top row) and diagonal Laplace with neural network and linearized predictive (bottom row).}
  \label{fig:diag_vs_full}
\end{figure}

This is perfectly consistent with our analysis as we show that the benefit of linearization primarily comes from the Jacobian term in the linearized predictive and the $\ggn$ sharing a kernel. Diagonal and other sparse approximations do not share the spectral properties with the Jacobian in the linearized predictive. \citet{george2018fast} note the potential advantages of having the spectrum of the approximate curvature more aligned with the true curvature. This suggests future directions to improve various approximations to the $\ggn$ by accounting for reparameterizations.

\paragraph{SWAG.}
Another baseline that can be explained using our method is the SWAG. It has been shown that in \citep{li2021happens} SGD steps close to the optimum can be decomposed into a normal space component and a tangent space component. In our terminology, this can be thought of as a diffusion step in the Kernel manifold and a diffusion step in the Non-kernel manifold. Hence it can be shown that SWAG roughly approximates a diffusion-based posterior. Hence our analysis can provide some theoretical grounding for heuristic methods like SWAG.

%% file: checklist.tex
\newpage
\section*{NeurIPS Paper Checklist}

\begin{enumerate}

\item {\bf Claims}
    \item[] Question: Do the main claims made in the abstract and introduction accurately reflect the paper's contributions and scope?
    \item[] Answer: \answerYes{} 
    \item[] Justification: The theoretical findings are states in abstract and introduction as clear as possible withouth first introducing all the notation of the paper. The experimental results match the claim and confirm the theoretical findings.
    \item[] Guidelines:
    \begin{itemize}
        \item The answer NA means that the abstract and introduction do not include the claims made in the paper.
        \item The abstract and/or introduction should clearly state the claims made, including the contributions made in the paper and important assumptions and limitations. A No or NA answer to this question will not be perceived well by the reviewers. 
        \item The claims made should match theoretical and experimental results, and reflect how much the results can be expected to generalize to other settings. 
        \item It is fine to include aspirational goals as motivation as long as it is clear that these goals are not attained by the paper. 
    \end{itemize}

\item {\bf Limitations}
    \item[] Question: Does the paper discuss the limitations of the work performed by the authors?
    \item[] Answer: \answerYes{} 
    \item[] Justification: The conclusion raises several points of concern.
    \item[] Guidelines:
    \begin{itemize}
        \item The answer NA means that the paper has no limitation while the answer No means that the paper has limitations, but those are not discussed in the paper. 
        \item The authors are encouraged to create a separate "Limitations" section in their paper.
        \item The paper should point out any strong assumptions and how robust the results are to violations of these assumptions (e.g., independence assumptions, noiseless settings, model well-specification, asymptotic approximations only holding locally). The authors should reflect on how these assumptions might be violated in practice and what the implications would be.
        \item The authors should reflect on the scope of the claims made, e.g., if the approach was only tested on a few datasets or with a few runs. In general, empirical results often depend on implicit assumptions, which should be articulated.
        \item The authors should reflect on the factors that influence the performance of the approach. For example, a facial recognition algorithm may perform poorly when image resolution is low or images are taken in low lighting. Or a speech-to-text system might not be used reliably to provide closed captions for online lectures because it fails to handle technical jargon.
        \item The authors should discuss the computational efficiency of the proposed algorithms and how they scale with dataset size.
        \item If applicable, the authors should discuss possible limitations of their approach to address problems of privacy and fairness.
        \item While the authors might fear that complete honesty about limitations might be used by reviewers as grounds for rejection, a worse outcome might be that reviewers discover limitations that aren't acknowledged in the paper. The authors should use their best judgment and recognize that individual actions in favor of transparency play an important role in developing norms that preserve the integrity of the community. Reviewers will be specifically instructed to not penalize honesty concerning limitations.
    \end{itemize}

\item {\bf Theory Assumptions and Proofs}
    \item[] Question: For each theoretical result, does the paper provide the full set of assumptions and a complete (and correct) proof?
    \item[] Answer: \answerYes{} 
    \item[] Justification: all theoretical statements are given a detailed derivation in the appendix.
    \item[] Guidelines:
    \begin{itemize}
        \item The answer NA means that the paper does not include theoretical results. 
        \item All the theorems, formulas, and proofs in the paper should be numbered and cross-referenced.
        \item All assumptions should be clearly stated or referenced in the statement of any theorems.
        \item The proofs can either appear in the main paper or the supplemental material, but if they appear in the supplemental material, the authors are encouraged to provide a short proof sketch to provide intuition. 
        \item Inversely, any informal proof provided in the core of the paper should be complemented by formal proofs provided in appendix or supplemental material.
        \item Theorems and Lemmas that the proof relies upon should be properly referenced. 
    \end{itemize}

    \item {\bf Experimental Result Reproducibility}
    \item[] Question: Does the paper fully disclose all the information needed to reproduce the main experimental results of the paper to the extent that it affects the main claims and/or conclusions of the paper (regardless of whether the code and data are provided or not)?
    \item[] Answer: \answerYes{} 
    \item[] Justification: We provide a detailed appendix describing the experimental setup. We will further release all code to reproduce the paper.
    \item[] Guidelines:
    \begin{itemize}
        \item The answer NA means that the paper does not include experiments.
        \item If the paper includes experiments, a No answer to this question will not be perceived well by the reviewers: Making the paper reproducible is important, regardless of whether the code and data are provided or not.
        \item If the contribution is a dataset and/or model, the authors should describe the steps taken to make their results reproducible or verifiable. 
        \item Depending on the contribution, reproducibility can be accomplished in various ways. For example, if the contribution is a novel architecture, describing the architecture fully might suffice, or if the contribution is a specific model and empirical evaluation, it may be necessary to either make it possible for others to replicate the model with the same dataset, or provide access to the model. In general. releasing code and data is often one good way to accomplish this, but reproducibility can also be provided via detailed instructions for how to replicate the results, access to a hosted model (e.g., in the case of a large language model), releasing of a model checkpoint, or other means that are appropriate to the research performed.
        \item While NeurIPS does not require releasing code, the conference does require all submissions to provide some reasonable avenue for reproducibility, which may depend on the nature of the contribution. For example
        \begin{enumerate}
            \item If the contribution is primarily a new algorithm, the paper should make it clear how to reproduce that algorithm.
            \item If the contribution is primarily a new model architecture, the paper should describe the architecture clearly and fully.
            \item If the contribution is a new model (e.g., a large language model), then there should either be a way to access this model for reproducing the results or a way to reproduce the model (e.g., with an open-source dataset or instructions for how to construct the dataset).
            \item We recognize that reproducibility may be tricky in some cases, in which case authors are welcome to describe the particular way they provide for reproducibility. In the case of closed-source models, it may be that access to the model is limited in some way (e.g., to registered users), but it should be possible for other researchers to have some path to reproducing or verifying the results.
        \end{enumerate}
    \end{itemize}

\item {\bf Open access to data and code}
    \item[] Question: Does the paper provide open access to the data and code, with sufficient instructions to faithfully reproduce the main experimental results, as described in supplemental material?
    \item[] Answer: \answerYes{} 
    \item[] Justification: We rely on established benchmark data. All code will be released under an open source software licence upon paper acceptance.
    \item[] Guidelines:
    \begin{itemize}
        \item The answer NA means that paper does not include experiments requiring code.
        \item Please see the NeurIPS code and data submission guidelines (\url{https://nips.cc/public/guides/CodeSubmissionPolicy}) for more details.
        \item While we encourage the release of code and data, we understand that this might not be possible, so “No” is an acceptable answer. Papers cannot be rejected simply for not including code, unless this is central to the contribution (e.g., for a new open-source benchmark).
        \item The instructions should contain the exact command and environment needed to run to reproduce the results. See the NeurIPS code and data submission guidelines (\url{https://nips.cc/public/guides/CodeSubmissionPolicy}) for more details.
        \item The authors should provide instructions on data access and preparation, including how to access the raw data, preprocessed data, intermediate data, and generated data, etc.
        \item The authors should provide scripts to reproduce all experimental results for the new proposed method and baselines. If only a subset of experiments are reproducible, they should state which ones are omitted from the script and why.
        \item At submission time, to preserve anonymity, the authors should release anonymized versions (if applicable).
        \item Providing as much information as possible in supplemental material (appended to the paper) is recommended, but including URLs to data and code is permitted.
    \end{itemize}

\item {\bf Experimental Setting/Details}
    \item[] Question: Does the paper specify all the training and test details (e.g., data splits, hyperparameters, how they were chosen, type of optimizer, etc.) necessary to understand the results?
    \item[] Answer: \answerYes{}
    \item[] Justification: 
    We specify the details about the experimental setup in the appendix. Furthermore, we also provide the code as supplemental material which contains the full details of the experiments.
    \item[] Guidelines:
    \begin{itemize}
        \item The answer NA means that the paper does not include experiments.
        \item The experimental setting should be presented in the core of the paper to a level of detail that is necessary to appreciate the results and make sense of them.
        \item The full details can be provided either with the code, in appendix, or as supplemental material.
    \end{itemize}

\item {\bf Experiment Statistical Significance}
    \item[] Question: Does the paper report error bars suitably and correctly defined or other appropriate information about the statistical significance of the experiments?
    \item[] Answer: \answerYes{} 
    \item[] Justification: We provide errors and/or standard deviations for presented results.
    \item[] Guidelines:
    \begin{itemize}
        \item The answer NA means that the paper does not include experiments.
        \item The authors should answer "Yes" if the results are accompanied by error bars, confidence intervals, or statistical significance tests, at least for the experiments that support the main claims of the paper.
        \item The factors of variability that the error bars are capturing should be clearly stated (for example, train/test split, initialization, random drawing of some parameter, or overall run with given experimental conditions).
        \item The method for calculating the error bars should be explained (closed form formula, call to a library function, bootstrap, etc.)
        \item The assumptions made should be given (e.g., Normally distributed errors).
        \item It should be clear whether the error bar is the standard deviation or the standard error of the mean.
        \item It is OK to report 1-sigma error bars, but one should state it. The authors should preferably report a 2-sigma error bar than state that they have a 96\% CI, if the hypothesis of Normality of errors is not verified.
        \item For asymmetric distributions, the authors should be careful not to show in tables or figures symmetric error bars that would yield results that are out of range (e.g. negative error rates).
        \item If error bars are reported in tables or plots, The authors should explain in the text how they were calculated and reference the corresponding figures or tables in the text.
    \end{itemize}

\item {\bf Experiments Compute Resources}
    \item[] Question: For each experiment, does the paper provide sufficient information on the computer resources (type of compute workers, memory, time of execution) needed to reproduce the experiments?
    \item[] Answer: \answerYes{} 
    \item[] Justification: these details are presented in the appendix.
    \item[] Guidelines:
    \begin{itemize}
        \item The answer NA means that the paper does not include experiments.
        \item The paper should indicate the type of compute workers CPU or GPU, internal cluster, or cloud provider, including relevant memory and storage.
        \item The paper should provide the amount of compute required for each of the individual experimental runs as well as estimate the total compute. 
        \item The paper should disclose whether the full research project required more compute than the experiments reported in the paper (e.g., preliminary or failed experiments that didn't make it into the paper). 
    \end{itemize}
    
\item {\bf Code Of Ethics}
    \item[] Question: Does the research conducted in the paper conform, in every respect, with the NeurIPS Code of Ethics \url{https://neurips.cc/public/EthicsGuidelines}?
    \item[] Answer: \answerYes{} 
    \item[] Justification: The work is predominantly theoretical and does not raise ethical concerns.
    \item[] Guidelines:
    \begin{itemize}
        \item The answer NA means that the authors have not reviewed the NeurIPS Code of Ethics.
        \item If the authors answer No, they should explain the special circumstances that require a deviation from the Code of Ethics.
        \item The authors should make sure to preserve anonymity (e.g., if there is a special consideration due to laws or regulations in their jurisdiction).
    \end{itemize}

\item {\bf Broader Impacts}
    \item[] Question: Does the paper discuss both potential positive societal impacts and negative societal impacts of the work performed?
    \item[] Answer: \answerNo{} 
    \item[] Justification: The paper is predominantly theoretical. However, uncertainty quantification is a potential remedy for several issues with currently deployed machine learning models, so there is a potential for positive societal impact.
    \item[] Guidelines:
    \begin{itemize}
        \item The answer NA means that there is no societal impact of the work performed.
        \item If the authors answer NA or No, they should explain why their work has no societal impact or why the paper does not address societal impact.
        \item Examples of negative societal impacts include potential malicious or unintended uses (e.g., disinformation, generating fake profiles, surveillance), fairness considerations (e.g., deployment of technologies that could make decisions that unfairly impact specific groups), privacy considerations, and security considerations.
        \item The conference expects that many papers will be foundational research and not tied to particular applications, let alone deployments. However, if there is a direct path to any negative applications, the authors should point it out. For example, it is legitimate to point out that an improvement in the quality of generative models could be used to generate deepfakes for disinformation. On the other hand, it is not needed to point out that a generic algorithm for optimizing neural networks could enable people to train models that generate Deepfakes faster.
        \item The authors should consider possible harms that could arise when the technology is being used as intended and functioning correctly, harms that could arise when the technology is being used as intended but gives incorrect results, and harms following from (intentional or unintentional) misuse of the technology.
        \item If there are negative societal impacts, the authors could also discuss possible mitigation strategies (e.g., gated release of models, providing defenses in addition to attacks, mechanisms for monitoring misuse, mechanisms to monitor how a system learns from feedback over time, improving the efficiency and accessibility of ML).
    \end{itemize}
    
\item {\bf Safeguards}
    \item[] Question: Does the paper describe safeguards that have been put in place for responsible release of data or models that have a high risk for misuse (e.g., pretrained language models, image generators, or scraped datasets)?
    \item[] Answer: \answerNA{} 
    \item[] Justification: we do not deem this to be relevant.
    \item[] Guidelines:
    \begin{itemize}
        \item The answer NA means that the paper poses no such risks.
        \item Released models that have a high risk for misuse or dual-use should be released with necessary safeguards to allow for controlled use of the model, for example by requiring that users adhere to usage guidelines or restrictions to access the model or implementing safety filters. 
        \item Datasets that have been scraped from the Internet could pose safety risks. The authors should describe how they avoided releasing unsafe images.
        \item We recognize that providing effective safeguards is challenging, and many papers do not require this, but we encourage authors to take this into account and make a best faith effort.
    \end{itemize}

\item {\bf Licenses for existing assets}
    \item[] Question: Are the creators or original owners of assets (e.g., code, data, models), used in the paper, properly credited and are the license and terms of use explicitly mentioned and properly respected?
    \item[] Answer: \answerYes{} 
    \item[] Justification: we only consider well-established benchmark data, which we cite appropriately.
    \item[] Guidelines:
    \begin{itemize}
        \item The answer NA means that the paper does not use existing assets.
        \item The authors should cite the original paper that produced the code package or dataset.
        \item The authors should state which version of the asset is used and, if possible, include a URL.
        \item The name of the license (e.g., CC-BY 4.0) should be included for each asset.
        \item For scraped data from a particular source (e.g., website), the copyright and terms of service of that source should be provided.
        \item If assets are released, the license, copyright information, and terms of use in the package should be provided. For popular datasets, \url{paperswithcode.com/datasets} has curated licenses for some datasets. Their licensing guide can help determine the license of a dataset.
        \item For existing datasets that are re-packaged, both the original license and the license of the derived asset (if it has changed) should be provided.
        \item If this information is not available online, the authors are encouraged to reach out to the asset's creators.
    \end{itemize}

\item {\bf New Assets}
    \item[] Question: Are new assets introduced in the paper well documented and is the documentation provided alongside the assets?
    \item[] Answer: \answerYes{} 
    \item[] Justification: we will release open source code for reproducing experiments upon paper acceptance. This is well documented.
    \item[] Guidelines:
    \begin{itemize}
        \item The answer NA means that the paper does not release new assets.
        \item Researchers should communicate the details of the dataset/code/model as part of their submissions via structured templates. This includes details about training, license, limitations, etc. 
        \item The paper should discuss whether and how consent was obtained from people whose asset is used.
        \item At submission time, remember to anonymize your assets (if applicable). You can either create an anonymized URL or include an anonymized zip file.
    \end{itemize}

\item {\bf Crowdsourcing and Research with Human Subjects}
    \item[] Question: For crowdsourcing experiments and research with human subjects, does the paper include the full text of instructions given to participants and screenshots, if applicable, as well as details about compensation (if any)? 
    \item[] Answer: \answerNA{} 
    \item[] Justification: this paper does not involve crowdsourcing or other forms of research involving human subjects.
    \item[] Guidelines:
    \begin{itemize}
        \item The answer NA means that the paper does not involve crowdsourcing nor research with human subjects.
        \item Including this information in the supplemental material is fine, but if the main contribution of the paper involves human subjects, then as much detail as possible should be included in the main paper. 
        \item According to the NeurIPS Code of Ethics, workers involved in data collection, curation, or other labor should be paid at least the minimum wage in the country of the data collector. 
    \end{itemize}

\item {\bf Institutional Review Board (IRB) Approvals or Equivalent for Research with Human Subjects}
    \item[] Question: Does the paper describe potential risks incurred by study participants, whether such risks were disclosed to the subjects, and whether Institutional Review Board (IRB) approvals (or an equivalent approval/review based on the requirements of your country or institution) were obtained?
    \item[] Answer: \answerNA{} 
    \item[] Justification: this paper does not involve crowdsourcing or other forms of research involving human subjects.
    \item[] Guidelines:
    \begin{itemize}
        \item The answer NA means that the paper does not involve crowdsourcing nor research with human subjects.
        \item Depending on the country in which research is conducted, IRB approval (or equivalent) may be required for any human subjects research. If you obtained IRB approval, you should clearly state this in the paper. 
        \item We recognize that the procedures for this may vary significantly between institutions and locations, and we expect authors to adhere to the NeurIPS Code of Ethics and the guidelines for their institution. 
        \item For initial submissions, do not include any information that would break anonymity (if applicable), such as the institution conducting the review.
    \end{itemize}

\end{enumerate}